\documentclass[onefignum,onetabnum]{siamonline171218}
\usepackage[utf8]{inputenc} 
\usepackage[T1]{fontenc}    
\usepackage{color}
\usepackage{makecell}
\usepackage{amsmath,amsfonts,amssymb}
\usepackage{mathtools}
\usepackage{pifont}
\usepackage{booktabs}       
\usepackage{amsfonts}       
\usepackage{nicefrac}       
\usepackage{microtype}      
\usepackage{lipsum,bm}
\usepackage{graphicx}
\usepackage{ragged2e}
\usepackage{multirow}
\usepackage{enumitem}
\setlist[enumerate]{leftmargin=.5in}
\setlist[itemize]{leftmargin=.5in}
\graphicspath{ {./images/} }
\def\etal{\textit{et al. }}

\newsiamthm{remark}{Remark}
\newsiamremark{hypothesis}{Hypothesis}
\crefname{hypothesis}{Hypothesis}{Hypotheses}
\newsiamthm{claim}{Claim}
\usepackage{bm}

\headers{Quasi-Conformal Convolution}{H. Zhang, T.L. Ip and L.M. Lui}

\title{Quasi-Conformal Convolution : A Learnable Convolution for Deep Learning on Simply Connected Open Surfaces
\thanks{Submitted to the editors DATE.
}}

\author{Han Zhang \thanks{Department of Mathematics, City University of Hong Kong, Hong Kong, China. (hzhang863-c@my.cityu.edu.hk)
}
\and Tsz Lok Ip \thanks{Department of Mathematics, Chinese University of Hong Kong, Hong Kong, China;(enochitl@link.cuhk.edu.hk)
}
\and Lok Ming Lui \thanks{Department of Mathematics, Chinese University of Hong Kong, Hong Kong, China;(lmlui@math.cuhk.edu.hk)
}}

\begin{document}
\maketitle
\begin{abstract}
Deep learning on non-Euclidean domains is important for analyzing complex geometric data that lacks common coordinate systems and familiar Euclidean properties. A central challenge in this field is to define convolution on domains, which inherently possess irregular and non-Euclidean structures.
In this work, we introduce Quasi-conformal Convolution (QCC), a novel framework for defining convolution on simply-connected open surfaces using quasi-conformal theories. Each QCC operator is linked to a specific quasi-conformal mapping, enabling the adjustment of the convolution operation through manipulation of this mapping. By utilizing trainable estimator modules that produce quasi-conformal mappings, QCC facilitates adaptive and learnable convolution operators that can be dynamically adjusted according to the underlying data structured on the surfaces. QCC unifies a broad range of spatially defined convolutions, facilitating the learning of tailored convolution operators on each underlying surface optimized for specific tasks. Building on this foundation, we develop the Quasi-Conformal Convolutional Neural Network (QCCNN) to address a variety of tasks related to geometric data.
We validate the efficacy of QCCNN through the classification of images defined on curvilinear simply-connected open Riemann surfaces, demonstrating superior performance in this context. Additionally, we explore its potential in medical applications, including craniofacial analysis using 3D facial data and lesion segmentation on 3D human faces, achieving enhanced accuracy and reliability.
\end{abstract}

\section{Introduction}
Deep learning methodologies have increasingly been applied to non-Euclidean domains, encompassing irregular structures such as graphs and manifolds. While traditional deep learning operates on grid-like data, such as images and sequences, many real-world problems, ranging from the analysis of social networks to the study of 3D shapes, require direct learning from data defined on non-Euclidean geometric structures. The ability to process and extract meaningful representations of data defined on these complex domains makes these approaches essential for advancing machine learning applications in various fields, including computer vision, medical imaging, and computer-aided design.

The non-Euclidean nature of such data presents several challenges, primarily due to the absence of familiar Euclidean properties such as common coordinate systems, vector space structures, and shift-invariance. As a result, fundamental operations like convolution, which are central to deep learning in Euclidean domains, are not well-defined in these non-Euclidean domains. Previous works have attempted to define convolution on manifolds using various techniques, including spectral methods and spatially defined patch-based constructions. Spectral methods operate in the frequency domain; while elegant, these approaches are computationally expensive, sensitive to the choice of basis, and lack spatial locality. In contrast, patch-based methods seek to define convolution in the spatial domain by transferring local patches onto tangent spaces or employing local coordinate systems. However, these methods often require manually defined metrics. The primary goal of our paper is to explore methods for adapting convolution to non-Euclidean domains, specifically to facilitate deep learning on Riemann surfaces embedded in 3D.

In this work, we introduce Quasi-conformal Convolution (QCC), a novel framework for defining convolution on {simply-connected open} surfaces based on quasi-conformal theories. Each QCC operator is associated with a specific quasi-conformal mapping, allowing the convolution operation to be adjusted through the manipulation of this mapping. Through a trainable module that generates data-responsive quasi-conformal mappings, QCC allows for the adjustment of convolution operators to align with the specific structures of the surfaces. QCC encompasses a broad range of spatially defined convolutions, enabling the learning of customized convolution operators optimized for specific tasks on each underlying surface. Building on this foundation, we develop the Quasi-Conformal Convolutional Neural Network (QCCNN) to handle tasks involving information defined on the Riemann surfaces. QCCNN learns the optimal convolution operator for each underlying Riemann surface based on the specific task, allowing for more accurate results when processing data defined on complex geometric structures. The idea of learning the optimal convolution operator is particularly essential in complex surfaces, where the functions on the surfaces are considered highly distorted compared to planar images.


To demonstrate the efficacy of the proposed QCC framework, we applied our approach to several tasks on {simply-connected open} surfaces. First, we tested our method on the classification of images on curvilinear Riemann surfaces, achieving outstanding results in accurately classifying these images. Next, we applied our framework to craniofacial analysis using 3D human face data, which exhibited promising performance. Additionally, we conducted lesion segmentation on 3D human facial images, and our framework significantly outperformed existing methods, providing more accurate and reliable segmentation of facial lesions. Finally, we performed self-ablation studies to assess the impact of various parameter choices on the performance of the proposed models, further validating the robustness and adaptability of our approach.

To summarize, the main contributions of this work are as follows:
\begin{itemize}
    \item We provide a systematic and comprehensive definition of manifold convolution and parametrized manifold convolution, along with a theoretical investigation of their relationships.
    \item {We propose Quasi-Conformal Convolution (QCC), a novel framework for defining convolution on simply-connected open surfaces. The QCC framework enables the Riemannian metric to be defined by a quasi-conformal mapping, providing a theoretical foundation for developing adaptive convolution on such surfaces.}
    \item {Building on Quasi-Conformal Convolution, we develop the QCC layer for deep learning on simply-connected open surfaces, enabling learnable convolution specifically tailored to the geometric structure of the input surface and the requirements of the given task.} 
    \item We develop QCC Neural Networks that incorporate the designed QCC layer, making the model versatile for tasks such as classification and segmentation. Experiments demonstrate its advantages over related approaches.
\end{itemize}

\section{Related Works}
\subsection{Geometric Learning}
{The main objective of this work is to explore a general and principled approach for defining convolution operations on manifold domains. Unlike conventional Euclidean convolution \cite{lecun1995convolutional}, where data lies on a regular grid, manifold data resides on curved, often non-Euclidean spaces such as surfaces or meshes. Learning over such non-Euclidean domains is commonly referred to as geometric learning \cite{bronstein2017geometric}.}

In the field of geometric modelling, Bronstein \etal introduced manifold convolution with geodesic patch operators, demonstrating its success in various applications~\cite{ masci2015geodesic}. Similarly, Boscaini \etal utilized an anisotropic heat kernel to define the convolution window, further contributing to the field~\cite{boscaini2016learning}. Other convolution definitions have also succeeded in registration tasks~\cite{bouritsas2019neural, gong2019spiralnet++}. Additionally, the MeshCNN framework by Hanocka \etal is noteworthy, as it redefined convolution using edges rather than vertices, offering a natural and straightforward approach to the concept~\cite{hanocka2019meshcnn}. Schonsheck \etal propose \cite{schonsheck2022parallel} Parallel Transport Convolution to enhance the translation invariance and allow the construction of compactly supported filters in manifold neural networks.

\subsection{Deformable Convolution}

{The goal of this work is to develop a method for constructing adaptive convolution on manifold data.} 
Deformable convolution has been proposed to address the limitations of traditional convolution operations in Euclidean image domains when applied to Convolutional Neural Networks (CNNs). Jeon et al. \cite{jeon2017active} proposed Active Convolution (AC), incorporating a trainable attention mechanism into the convolution process. Related approaches are Spatial Transformer Network (STN)
and its variants \cite{jaderberg2015spatial,zhang2024learning,zhang2025deformation}, which introduce a learnable transformation module that can warp the input feature map
based on a set of learnable parameters. Subsequently, Dai et al. proposed Deformable Convolution (DCN) \cite{dai2017deformable}, which employs learnable offsets in the convolutional kernel to enable spatial adaptation. While DCN is a significant advancement, it struggles with handling large deformations and achieving occlusion invariance. To address these limitations, variants such as Deformable Convolution v2 (DCNv2) \cite{zhu2019deformable} and Deformable RoI Pooling (DRoIPool) \cite{dai2017deformable} were introduced. However, Luo et al. \cite{luo2016understanding} demonstrated that not all pixels contribute equally to the final DCN output. {These methods adjust how convolution is performed on data in Euclidean domains. In this work, we focus on developing a framework for learnable convolution on manifold data defined on simply-connected open Riemann surfaces, specifically tailored to the input data, geometric structure of the underlying surface and the requirements of the given task.}

\subsection{Computational Quasi-Conformal Mapping}
{To define deformable convolution on surfaces, it is essential to control the deformation. Computational quasi-conformal mapping offers a powerful tool for managing this, with successful applications in image science~\cite{lam2014landmark} and surface processing~\cite{levy2002least,gu2004genus}}. Benefitting from the Beltrami representation, the mapping between two different domains can preserve good geometric properties like bijectivity and smoothness, through controlling the Beltrami coefficients with such a representation of mappings. Driven by the motivation to preserve different geometric information, ways of parameterization methods are proposed~\cite{gu2003global}. Such convenient representations are also popular and succeed in the computational fabrication community~\cite{Soliman:2018:OCS,Crane:2013:RFC,panetta2019x}. With the capability to handle large deformations, the quasi-conformal method also succeeds in registration~\cite{lam2014landmark,choi2015fast}, restoration~\cite{zhang2025deformation}, and segmentation with topology- and convexity prior~\cite{zhang2025qis,zhang2024learning}. In \cite{zhang2022nondeterministic,zhang2022new}, quasi-conformality is used for deformation analysis with uncertainties to study medical images for disease analysis.

\section{Mathematical Background}
\subsection{Quasi-Conformal Geometry}

\begin{definition}[Quasi-conformal map]
A quasi-conformal map is a map $f: \mathbb{C} \rightarrow \mathbb{C}$ that satisfies the Beltrami equation
\begin{equation}
\frac{\partial f}{\partial \bar{z}}=\mu(z) \frac{\partial f}{\partial z}
\label{eq:beleq}
\end{equation}
for some complex-valued function named as Beltrami coefficient $\mu$ satisfying $\|\mu\|_{\infty}<1$ and $\frac{\partial f}{\partial z}$ is non-vanishing almost everywhere. The complex partial derivatives are given by
\begin{equation}
\frac{\partial f}{\partial z}:=\frac{1}{2}\left(\frac{\partial f}{\partial x}-i \frac{\partial f}{\partial y}\right) 
\quad \text{ and } \quad 
\frac{\partial f}{\partial \bar{z}}:=\frac{1}{2}\left(\frac{\partial f}{\partial x}+i \frac{\partial f}{\partial y}\right).
\end{equation}

\begin{figure}
    \centering
    \includegraphics[width=0.5\textwidth]{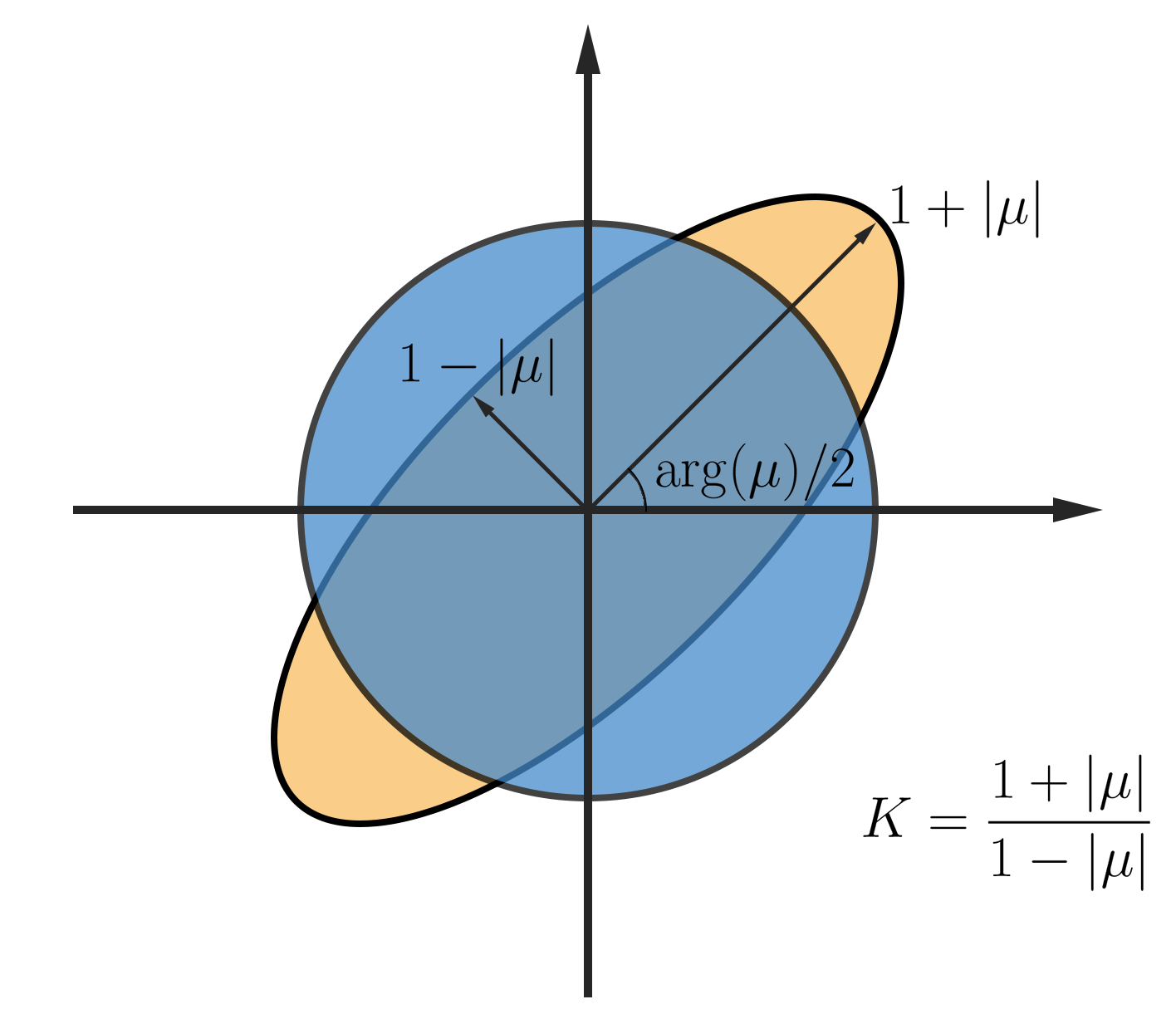}
    \caption{Illustration of how the Beltrami coefficient measures the conformality distortion of a quasi-conformal map}
    \label{fig:qcmap}
\end{figure}

\end{definition}
$\mu$ is the Beltrami representation, which is also called the Beltrami coefficient, of the quasi-conformal map $f$. It's worthy to mention that $\mu$ is a measure of non-conformality. Particularly, for a point $p$, the associated quasi-conformal map $f$ is conformal around a small neighbourhood of $p$ if $\mu(p)=0$. In this case, Equation \ref{eq:beleq} becomes the Cauchy-Riemann equation. This can also illustrate that conformality analysis of a quasi-conformal map $f$ can be simplified into the analysis of its associated Beltrami coefficient $\mu$. Infinitesimally, such a map $f$ can be rewritten as follows in a local neighbourhood around a point $p$:
\begin{equation}
\begin{aligned}
f(z) &=f(p)+f_{z}(p) z+f_{\bar{z}}(p) \bar{z} \\
&=f(p)+f_{z}(p)(z+\mu(p) \bar{z}).
\end{aligned}
\end{equation}
This further enhances our discussion before that $f$ is conformal when $\mu(p) = 0$. To explain the equation above, $f(p)$ is a translation, while $f_z(p)$ is a dilation. Since both of them are conformal, all the non-conformality of $f$ is brought by $D(z)=z+\mu(p) \bar{z}$. Hence, the Beltrami coefficient $\mu$ actually encodes the conformality of $f$. Analyzing quasi-conformal $f$ is equivalent to that for its associated Beltrami coefficient $\mu$. To be detail, the angle of maximal magnification is $\arg (\mu(p)) / 2$ with magnifying factor $1 + |\mu(p)|$; the maximal shrinking is the orthogonal angle $(\arg (\mu(p))-\pi) / 2$ with shrinking factor $1 - |\mu(p)|$. 

The maximal quasi-conformal dilation of $f$ is given by
\begin{equation}
K=\frac{1+\|\mu\|_{\infty}}{1-\|\mu\|_{\infty}}.
\end{equation}
Figure \ref{fig:qcmap} illustrates the geometry of a quasi-conformal map.

Another important relationship between a map and its Beltrami coefficients is the diffeomorphism property. By a norm constraint on $\mu$, the bijectivity of $f$ can be preserved which is explained by the following theory.

\begin{theorem}
If $f: \mathbb{C} \rightarrow \mathbb{C}$ is a $C^{1}$ map. Define 
\begin{equation}
\mu=\frac{\partial f}{\partial \bar{z}} / \frac{\partial f}{\partial z}.
\end{equation}
If $\mu$ satisfies $\left\|\mu_{f}\right\|_{\infty}<1$, then $f$ is bijective.
\end{theorem}

\begin{theorem}[Measurable Riemann mapping theorem \cite{gardiner2000quasiconformal}]
\label{them:RiemannMapping}
Suppose $\mu: \mathbb{C}\rightarrow\mathbb{C}$ is Lebesgue measurable satisfying $\|\mu\|_{\infty}<1$, then there exists a quasi-conformal mapping $f:\mathbb{C}\rightarrow \mathbb{C}$ in the Sobolev space $W^{1,2}$ that satisfies the Beltrami equation in the distribution sense. Furthermore, assuming that the mapping is stationary at $0, 1$ and $\infty$, then the associated quasi-conformal mapping $f$ is uniquely determined.
\end{theorem}

\begin{figure}
    \centering
    \includegraphics[width=0.8\textwidth]{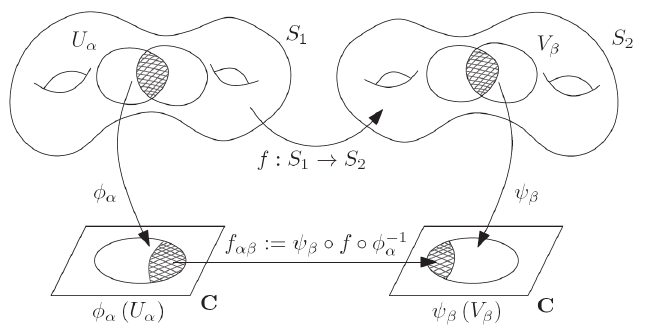}
    \caption{Illustration of quasi-conformal mapping between Riemann surfaces.}
    \label{fig:bcdifferential}
\end{figure}

The Beltrami coefficient of a composition of quasi-conformal maps is related to the Beltrami coefficients of the original maps. Suppose $f: \Omega \rightarrow f(\Omega)$ and $g: f(\Omega) \rightarrow \mathbb{C}$ are two quasi-conformal maps with Beltrami coefficients $\mu_f$ and $\mu_g$ correspondingly. The Beltrami coefficient of the composition map $g \circ f$ is given by
$$
\mu_{g \circ f}=\frac{\mu_f+\frac{\overline{f_z}}{f_z}\left(\mu_g \circ f\right)}{1+\frac{\overline{f_z}}{f_z} \overline{\mu_f}\left(\mu_g \circ f\right)} .
$$

Quasi-conformal maps can also be defined between two Riemann surfaces. In this case, the Beltrami differential is used. A Beltrami differential $\mu(z) \frac{d\overline{z}}{d z}$ on a Riemann surface $S$ is an assignment to each chart ($U_\alpha, \phi_\alpha$) of an $L_{\infty}$ complex-valued function $\mu_\alpha$, defined on local parameter $z_\alpha$ such that
$$
\mu_\alpha \frac{d \overline{z_\alpha}}{d z_\alpha}=\mu_\beta \frac{d \overline{z_\beta}}{d z_\beta},
$$
on the domain which is also covered by another chart $\left(U_\beta, \phi_\beta\right)$. Here, $\frac{d z_\beta}{d z_\alpha}=\frac{d}{d z_\alpha} \phi_{\alpha \beta}$ and $\phi_{\alpha \beta}=\phi_\beta \circ \phi_\alpha$. An orientation preserving diffeomorphism $f: M \rightarrow N$ is called quasi-conformal associated with $\mu(z) \frac{d z}{d z}$ if for any chart ($U_\alpha, \phi_\alpha$) on $M$ and any chart $\left(U_\beta, \psi_\beta\right)$ on $N$, the mapping $f_{\alpha \beta}:=\psi_\beta \circ f \circ f_\alpha^{-1}$ is quasi-conformal associated with $\mu_\alpha \frac{d \overline{z}}{d z_\alpha}$. Readers are referred to \cite{gardiner2000quasiconformal,lehto1973quasiconformal} for more details about quasi-conformal theories.
\section{Adaptive Convolution on Riemannian Manifolds}
\label{sec:convolution}

Convolution is a fundamental mathematical operator in mathematics, physics, and engineering. It combines two functions to illustrate how the characteristics of one function are modified by the other. The necessity of defining convolution stems from its applications in various contexts. In this section, we will define convolution on Riemann surfaces, starting with an investigation of convolution on general manifolds and progressing to the definition of convolution on simply connected surfaces through Quasi-conformal Convolution (QCC), which is applicable to many real-world situations.

\subsection{Convolution on Riemannian \textit{n}-manifold}

\subsubsection{Convolution on Manifold}
Before providing a definition of convolution on manifolds, we first examine the standard convolution operation in Euclidean space, as outlined in the following definition.

\begin{definition}[Convolution]
For two functions $h, k : \mathbb{R}^n \to \mathbb{R}$, the convolution of $h$ and $k$ is defined as:
\begin{equation}
(h \ast k)(x) = \int_{\mathbb{R}^n} h(y) k(x - y) \, dy,
\end{equation}
where:
\begin{itemize}
    \item $x \in \mathbb{R}^n$ is the point at which the convolution is evaluated,
    \item $y \in \mathbb{R}^n$ is the integration variable,
    \item $k(x - y)$ translates $k$ to align it with $h$.
\end{itemize}
\end{definition}

Defining convolution on manifolds or surfaces is more complex than in Euclidean space due to the absence of a global linear structure on manifolds. In the Euclidean case, convolution involves translating the kernel function $k$ using the displacement vector $x-y$. However, this approach does not directly apply to manifolds. A suitable notion of displacement must be established before performing convolution on manifolds.

\begin{definition}[Displacement function and displacement vector]
    Let $\mathcal{M}$ be a Riemannian $n$-manifold, and $U \subseteq \mathcal{M}$ be a subset of $\mathcal{M}$. A function $d: U\times U \to \mathbb{R}^n$ is a displacement function on $U$ if it satisfies the following properties:
    \begin{enumerate}
        \item For all $p, q \in U$, $d(p,q) = 0$ if and only if $p=q$.
        \item For all $p, q, r \in U$, $d(p,r) = d(p,q) + d(q,r)$.
    \end{enumerate}
    Then, the vector $d(p, q)$ is referred to as the displacement vector from $p$ to $q$.
    
    Moreover, if the functions $d(\cdot, q_0)$ and $d(p_0, \cdot)$ are orientation-preserving homeomorphisms from $U$ to subsets of $\mathbb{R}^n$ depending on $p_0$ and $q_0$ for all fixed $p_0, q_0 \in U$, then the displacement function $d$ is said to be regular.
\end{definition}


The above displacement function is introduced to replace the standard expression $ x-y $ in $ \mathbb{R}^n $. This displacement function closely mimics the fundamental property of translation symmetry \cite{schonsheck2022parallel}. According to our definition, the displacement between any two distinct points on a manifold is always nonzero, and the sum of vectors along a path aligns precisely with the vector directly connecting the path’s endpoints. By adhering to these properties, the kernel can be effectively translated across different points on the manifold.

With the displacement function $d$ on $\mathcal{M}$, we are now ready to give a general definition of manifold convolution.

\begin{definition}[Manifold convolution]\label{def:mani_conv}
Let $\mathcal{M}$ be a Riemannian $n$-manifold with a metric $g$, and let $h: \mathcal{M} \to \mathbb{R}$ be a manifold function with a kernel function $k: \mathbb{R}^{n}\to \mathbb{R}$. The convolution of $h$ and $k$ on $\mathcal{M}$ is defined as:

\begin{equation}\label{eq:mani_conv}
(h \ast_{\mathcal{M},d,g} k)(p) = \int_{\mathcal{M}} h(q) k(d(p, q)) \, dq,
\end{equation}

where:
\begin{itemize}
    \item $p, q \in \mathcal{M}$,
    \item $d: \mathcal{M}\times\mathcal{M} \to \mathbb{R}^n $ is a global displacement function on $\mathcal{M}$.
\end{itemize}
For simplicity, we denote $*_{\mathcal{M}, d,g}$ as $*_{d,g}$.
Moreover, $*_{d,g}$ is said to be regular if the displacement function $d$ is regular.

\end{definition}



Note that the manifold convolution above is not commutative as the manifold function $h$ is defined on $\mathcal{M}$ while kernel function $k$ is defined on $\mathbb{R}^n$. This lack of commutativity does not hinder the definition of the convolution operation for deep learning tasks on Riemann surfaces. 


Before we proceed to the next section, we shall emphasize that the convolution operator $*_{d,g}$ does depend on the Riemannian metric $g$ of $\mathcal{M}$ due to the standard definition of integration on manifolds used in Equation \ref{eq:mani_conv}. More details are discussed in the following remark.
\begin{remark}\label{rm:volume_form}
    Let $\mathcal{M}$ be a Riemannian $n$-manifold and let $h: \mathcal{M} \to \mathbb{R}$ be a manifold function. 
    Suppose $\mathcal{M}$ is covered by a collection of coordinate charts $\{(U_\alpha, \phi_\alpha)\}$, and let $\{\psi_\alpha\}$ be a partition of unity subordinate to the cover $\{U_\alpha\}$. In local coordinates, the {\it volume form $d\nu$ induced by $g$} is given by:
\begin{equation}
d\nu =  \sqrt{\det(g(x))} \, dx^1 \wedge dx^2 \wedge \cdots dx^n,
\end{equation}
\noindent where $g(x)=(g_{ij}(x))_{1\leq i,j\leq n}$. The manifold integral of $h$ is then defined as
    \begin{equation}
    \int_\mathcal{M} h \, d\nu = \sum_\alpha \int_{U_\alpha} \psi_\alpha(x) h(\phi_\alpha^{-1}(x)) \sqrt{\det(g(x))} \, dx.
    \end{equation}
  \end{remark}


\subsubsection{Convolution on Manifold via Parameterization}

The concept of the convolution operation on a manifold is not well-established, primarily due to the curvature of the manifold. In the Euclidean domain, the plain convolution operator involves shifting the kernel function. This shifting is straightforward in Euclidean space, where the geometry is flat and the displacement from one point to another is well-defined. In this subsection, we introduce the idea of defining convolution on a Riemann surface on its 2D parametric flat domain.

\begin{definition}[Parametrized manifold convolution]
\label{def:parameterizedconv}
    Let $\mathcal{M}$ be a Riemannian $n$-manifold, and let $h: \mathcal{M} \to \mathbb{R}$ be a manifold function with a kernel function $k: \mathbb{R}^{n}\to \mathbb{R}$. Suppose there exists a bijective parametrization $\phi:\Omega \to \mathcal{M}$, where $\Omega \subset \mathbb{R}^n$. The parametrized manifold convolution of $h$ and $k$ with respect to $\phi$ is defined as:
    \begin{equation}
        (h *_\phi k)(p) = \int_\Omega h(\phi(y))k(\phi^{-1}(p)-y)dy.
    \end{equation}
    
    
    \label{them:parameterconv}
\end{definition}

The definition above simplifies the computation of manifold convolution and provides a more intuitive viewpoint by performing convolution on an Euclidean domain. To explore the relationship between parameterized manifold convolution and manifold convolution, the following lemma is essential to demonstrate that a displacement function on a manifold gives rise to a surface parameterization, enabling us to define the parameterized manifold convolution.



\begin{lemma}\label{disp2phi}
     Let $\mathcal{M}$ be a Riemannian $n$-manifold. If there exists a displacement function $d:\mathcal{M} \times \mathcal{M} \to \mathbb{R}^n$, then there exists a bijective parametrization $\phi: \Omega \to \mathcal{M}$ such that
     \begin{equation}
         d(p,q) = \phi^{-1}(p)-\phi^{-1}(q)
     \end{equation}
     for all $p, q \in \mathcal{M}$, where $\Omega \subset \mathbb{R}^n$.

    Conversely, if there exists a bijective parametrization $\phi: \Omega \to \mathcal{M}$, where $\Omega \subset \mathbb{R}^n$, then a function $d:\mathcal{M} \times \mathcal{M} \to \mathbb{R}^n$ defined by 
    \begin{equation}
         d(p,q) = \phi^{-1}(p)-\phi^{-1}(q)
    \end{equation}
    for all $p, q \in \mathcal{M}$, is a displacement function on $\mathcal{M}$.
\end{lemma}

\begin{proof}
    $(\Rightarrow)$ Suppose there exists a displacement function $d:\mathcal{M} \times \mathcal{M} \to \mathbb{R}$. Pick any $p_0 \in \mathcal{M}$ and let $\Omega = \{d(p, p_0): p \in \mathcal{M}\} \subset \mathbb{R}^n$. Define $\tilde{\phi}: \mathcal{M} \to \Omega$ by $\tilde{\phi}(p) = d(p, p_0)$. Then
    \begin{equation}
        d(p, q) = d(p, p_0) - d(q, p_0) = \tilde{\phi}(p)-\tilde{\phi}(q).
    \end{equation}
    Note that this equation shows that $\tilde{\phi}$ is injective as $d(p,q)=0$ implies $p=q$, and the surjectivity of $\tilde{\phi}$ is guaranteed by the definition of $\Omega$, $\tilde{\phi}$ is therefore bijective. 
    Thus $\phi = \tilde{\phi}^{-1}$ is the desired bijective parametrization.
    
    $(\Leftarrow)$ Check that for all $p, q, r \in \mathcal{M}$,
    $$d(p, q) = 0 \iff \phi^{-1}(p) = \phi^{-1}(q) \iff p=q$$
    and
    $$d(p, r) = \phi^{-1}(p) - \phi^{-1}(r) = (\phi^{-1}(p) - \phi^{-1}(q)) + (\phi^{-1}(q) - \phi^{-1}(r)) = d(p, q) + d(q, r).$$
    Therefore, $d$ is a displacement function for $\mathcal{M}$.
\end{proof}




\begin{theorem}\label{thm:equiv_conv}
    Let $\mathcal{M}$ be a Riemannian $n$-manifold and $d$ be a displacement function on $\mathcal{M}$. Then there exists a bijective parametrization $\phi: \Omega \to \mathcal{M}$, where $\Omega \subset \mathbb{R}^n$, along with a metric $g$ of $\mathcal{M}$, such that $*_{d, g} = *_\phi$. Conversely, for any bijective parametrization $\phi: \Omega \to \mathcal{M}$, there exists a displacement function $d$ on $\mathcal{M}$ and a metric $g$ of $\mathcal{M}$ such that $*_{d, g} = *_\phi$.
\end{theorem}

\begin{proof}
    ($\Rightarrow$) Suppose $*_\mathcal{M} = *_{d,g}$ for some displacement function $d$ on $\mathcal{M}$. By Lemma \ref{disp2phi}, there exists a bijective parametrization $\phi: \Omega \to \mathcal{M}$ such that $$d(p,q) = \phi^{-1}(p)-\phi^{-1}(q)$$ for all $p, q \in \mathcal{M}$, where $\Omega \subset \mathbb{R}^n$. Consider $\phi^{-1}$ as the coordinate chart function and define the Riemannian metric $g = (\phi^{-1})^*g_{\mathbb{R}^n}$ as a pullback metric of $\mathcal{M}$, where $g_{\mathbb{R}^n}$ is the standard Euclidean metric. The volume form $d\nu$ can be obtained by Remark \ref{rm:volume_form}. 
    
    Since the distance $dg(p,q) = |\phi^{-1}(p)-\phi^{-1}(q)|$ for all $p,q \in \mathcal{M}$, $\phi^{-1}$ is an isometric mapping with respect to the metric $g$ and $\det(D\phi^{-1}) = 1$, therefore $dy = |\det(D\phi^{-1})| \, dq = dq$ for $\phi(y) = q$. For any manifold function $h:\mathcal{M} \to \mathbb{R}$, kernel function $k: \mathbb{R}^n \to \mathbb{R}$ and $p \in \mathcal{M}$, we now have 
    \begin{equation}
        \begin{aligned}
            (h *_\phi k)(p) &= \int_\Omega h(\phi(y))k(\phi^{-1}(p)-y)dy \\
            &= \int_\mathcal{M} h(q)k(\phi^{-1}(p)-\phi^{-1}(q))dq \\
            &= \int_\mathcal{M} h(q)k(d(p,q))dq \\
            &= (h *_{d,g} k)(p).
        \end{aligned}
        \label{eq:para2mani}
    \end{equation}
    Therefore $*_{d, g} = *_\phi$.

    ($\Leftarrow$) By Lemma \ref{disp2phi}, the function $d(p,q) = \phi^{-1}(p)-\phi^{-1}(q)$ is a displacement function. Therefore, the result follows similarly.
\end{proof}




Theorem \ref{thm:equiv_conv} establishes that the space of manifold convolutions, equipped with a displacement function $d$ and an associated metric $g$, is equivalent to the space of parameterized manifold convolutions. Hence, the set of all parameterized manifold convolutions encompasses a substantial portion of manifold convolutions, suggesting that many spatially defined manifold convolutions can be effectively represented within this parameterized framework. To conclude this section, we will investigate the regularity properties of both types of convolutions.


\begin{corollary}
    Let $\mathcal{M}$ be a Riemannian $n$-manifold equipped with a displacement function $d$ and a metric $g$, is parametrized by a function $\phi$ such that $*_{d,g} = *_\phi$. Then the following statements are equivalent:
    \begin{enumerate}
        \item $d$ is regular.
        \item $\phi$ is an orientation-preserving homeomorphism.
        \item $*_{d,g} = *_\phi$ is regular.
    \end{enumerate}
    \label{them:manifoldparametrized}
\end{corollary}

\begin{proof}
    Note that (1) and (3) are equivalent by definition, and (2) implies (1) as $*_{d,g} = *_\phi$ implies $d(p,q) = \phi^{-1}(p) - \phi^{-1}(q)$, which is regular if $\phi$ is an orientation-preserving homeomorphism.
    
    (1) $\Rightarrow$ (2): Suppose $d$ is regular, then using the construction of $\phi$ in the previous theorem, we immediately see $\phi$ is an orientation-preserving homeomorphism. Assume $*_{d,g} = *_{\phi_1}$ for some bijective parametrization $\phi_1: \Omega_1 \to \mathcal{M}$, where $\Omega_1 \subset \mathbb{R}^n$, then for all $p, q \in \mathcal{M}$,
    \begin{equation}
    \begin{aligned}
        &\phi^{-1}(p) - \phi^{-1}(q) = \phi_1^{-1}(p) - \phi_1^{-1}(q) \\
        \Rightarrow \quad &\phi^{-1}(p) - \phi_1^{-1}(p) = \phi^{-1}(q) - \phi_1^{-1}(q) \\
        \Rightarrow \quad &\phi^{-1} - \phi_1^{-1} \equiv c \qquad \text{for some constant vector } c \in \mathbb{R}^n.
    \end{aligned}
    \end{equation}
    Hence $\phi_1$ is also an orientation-preserving homeomorphism. 
\end{proof}

\subsection{Convolution on Riemann Surfaces}
In this work, we focus our problems on simply connected open surfaces embedded in $\mathbb{R}^3$, now we will move on to describe how we can define convolution on Riemann surfaces.

\subsubsection{Conformal Convolution}
A natural and useful approach to produce a parametrized manifold convolution is to employ a conformal parametrization of the manifold. A conformal parametrization $\phi$ is a map from a domain $\Omega \subset \mathbb{R}^2$ to the manifold $\mathcal{M} \subset \mathbb{R}^3$ that preserves angles. By mapping the curved surface to a flat Euclidean space, we can take advantage of the well-established theory of Euclidean convolutions. 

To define the convolution of a manifold function $h : \mathcal{M} \to \mathbb{R}$ on a $2$-manifold $\mathcal{M}$ and a kernel function $k : \mathbb{R}^2 \to \mathbb{R}$, we begin by pulling the function back to the parameter domain $\Omega$ using the conformal parametrization $\phi$. This transforms the problem of manifold convolution into a more manageable Euclidean convolution problem. Specifically, for each function $h$ defined on the surface and $k$ defined on $\mathbb{R}^2$, we define the pullback function $\tilde{h}$ on the flat domain $\Omega$ as follows:
\begin{equation}
\tilde{h}:=\phi^* h = h \circ \phi.
\end{equation}
The pullback functions $\tilde{h}$ are now defined on the Euclidean space $\Omega \subset \mathbb{R}^2$, where the convolution operation can be performed using the standard Euclidean formulation.

The convolution of $\tilde{h}$ and $k$ in the Euclidean domain $\Omega$ is then defined as:
\begin{equation}
(\tilde{h} \ast k)(x) = \int_{\Omega} \tilde{h}(y) k(x - y) \, dy,
\end{equation}
where the integral is taken over the domain $\Omega \subset \mathbb{R}^2$. This is the classical convolution in Euclidean space, which is computationally efficient and well-understood.

Once the convolution is computed in the Euclidean domain, the result must be mapped back to the manifold $\mathcal{M}$. This is done by applying the inverse of the conformal parametrization $\phi^{-1}$, which maps the Euclidean result back onto the manifold's coordinates. The convolution on the manifold is then given by:
\begin{equation}
(h \ast_\phi k)(p) = (\tilde{h} \ast k)(\phi^{-1}(p)),
\end{equation}
where $p \in \mathcal{M}$ is a point on the manifold, and $\phi^{-1}(p)$ returns the corresponding point in the parameter domain $\Omega$.

Then, the manifold convolution is reduced to a Euclidean convolution performed in the parametrized space $\Omega$, followed by a pullback to the manifold using the inverse of the conformal map. The final formal definition is then written as:
{
\begin{definition}[Conformal Convolution]
Let $\mathcal{M} \subset \mathbb{R}^3$ be a 2-manifold, and let $h: \mathcal{M} \to \mathbb{R}$ be a manifold function with a kernel function $k: \mathbb{R}^{n}\to \mathbb{R}$. Suppose there exists a conformal parametrization $\phi: \Omega \to \mathcal{M}$, where $\Omega \subset \mathbb{R}^2$. The conformal convolution of $h$ and $k$ with respect to $\phi$ is defined as:
\begin{equation}
    (h \ast_\phi k)(p) \coloneqq \int_{\Omega} h(\phi(y)) k(\phi^{-1}(p) - y) dy.
\end{equation}
\label{def:conformalconv}
\end{definition}
}
It is also important to note that, according to the Riemann Mapping Theorem (Theorem \ref{them:RiemannMapping}), the conformal parameterization is unique when three points are fixed. In practice, by mapping the surface to a disk or rectangle while fixing the boundary points, the parametrization becomes determined. 

{
Importantly, conformal convolution is a special case of the parametrized convolution defined in Definition~\ref{def:parameterizedconv}. According to Theorem~\ref{thm:equiv_conv}, it can also be expressed in the form of a manifold convolution, as summarized in the following remark:
\begin{remark}
Under the condition of Definition~\ref{def:conformalconv}, the conformal convolution of $h$ and $k$ can be written in the manifold convolution form given in Definition~\ref{def:mani_conv} as:
\begin{equation}   
    (h \ast_\phi k)(p) = (h \ast_{d,g} k)(p).
\end{equation}
which is followed from \eqref{eq:para2mani}, where:
\begin{itemize}
    \item $p,q \in \mathcal{M}$,
    \item $x=\phi^{-1}(p), y=\phi^{-1}(q) \in \Omega$,
    \item $d(p,q) = \phi^{-1}(p) - \phi^{-1}(q)$ is the displacement function,
    \item $g = (\phi^{-1})^* g_{\mathbb{R}^2}$ is the Riemannian metric of $\mathcal{M}$.
\end{itemize}    
Moreover, by introducing the pullback function $\tilde{h} = h \circ \phi$, the convolution can be further reduced to a standard 2D Euclidean convolution:
\begin{equation}
\begin{aligned}
    (h \ast_\phi k)(p) 
    &= \int_{\Omega} h(\phi(y)) k(\phi^{-1}(p) - y) \\
    &= \int_{\Omega} \tilde{h}(y) k(x - y) dy \\
    &= (\tilde{h} \ast k)(x).
\end{aligned}
\end{equation}
\end{remark}
}

Although conformal parametrization retains nice geometric properties once the surface is mapped to the 2D domain, there is no evidence that it is the best parametrization to define a parametrized convolution. As a special case of parametrized convolution, conformal convolution is too restrictive for advanced usage, such as when implemented into deep learning tasks. Therefore, a much more flexible convolution, namely Quasi-Conformal Convolution, will be introduced in the next subsection.

\subsubsection{Quasi-Conformal Convolution}

The proposed Quasi-Conformal Convolution (QCC) is to define convolution operation on manifolds using quasi-conformal mappings. Quasi-conformal theory offers a mathematically robust framework for studying deformations between surfaces while preserving local geometric structures. By leveraging this theory, QCC extends the convolution operation to non-Euclidean domains such as manifolds, enabling deep learning methods to process irregular and geometrically distorted data.

Here, similar to how we define parametrized convolution, we have the following definition for quasi-conformal convolution.




{
\begin{definition}[Quasi-conformal Convolution]
Let $\mathcal{M} \subset \mathbb{R}^3$ be a 2-manifold, and let $h: \mathcal{M} \to \mathbb{R}$ be a manifold function with a kernel function $k: \mathbb{R}^{n}\to \mathbb{R}$. Suppose there exists a conformal parametrization $\phi: \Omega \to \mathcal{M}$ and a quasi-conformal mapping $f: \Omega \to \Omega$, where $\Omega \subset \mathbb{R}^2$. The quasi-conformal convolution of $h$ and $k$ with respect to $\phi$ and $f$ is defined as:
\begin{equation}
    (h \ast_{\phi, f} k)(p)
    = \int_{\Omega}h(\phi(f^{-1}(y')))k(f\circ\phi^{-1}(p) - y') dy'.
\end{equation}
\label{def:qc_conv}
\end{definition}
}
{
Similar to conformal convolution, quasi-conformal convolution is a special case of the parametrized convolution defined in Definition~\ref{def:parameterizedconv} by having the parameterization as $\phi \circ f^{-1}$. Thus, according to Theorem~\ref{thm:equiv_conv}, we could have the following remark:
\begin{remark}
Under the condition of Definition~\ref{def:qc_conv}, the quasi-conformal convolution of $h$ and $k$ can be written in the manifold convolution form given in Definition~\ref{def:mani_conv} as:
\begin{equation}   
\begin{aligned}
    (h \ast_{\phi, f} k)(p)
    &= \int_{\Omega}h(\phi(f^{-1}(y')))k(f\circ\phi^{-1}(p) - y') dy'\\
    &= \int_{\mathcal{M}}h(q)k(f\circ\phi^{-1}(p) - f\circ\phi^{-1}(q)) dq \\
    &= (h \ast_{d,g }k)(p).
    \label{eq:qcc_rewrite}  
\end{aligned}
\end{equation}
where:
\begin{itemize}
    \item $p,q \in \mathcal{M}$,
    \item $x=\phi^{-1}(p), y=\phi^{-1}(q) \in \Omega$,
    \item $x'=f(x), y'=f(y) \in \Omega$,    
    \item $d(p,q) = f\circ\phi^{-1}(p) - f\circ\phi^{-1}(q)$ is the displacement function,
    \item $g = (f \circ \phi^{-1})^* g_{\mathbb{R}^2}$ is the Riemannian metric of $\mathcal{M}$.
\end{itemize}    
Besides, it can also be rewritten into a general 2D convolution by having the pullback function $\tilde{h} = h \circ \phi$, and the transformed pullback function $h^\#=\tilde{h}\circ f^{-1}$, which leads to
\begin{equation}
\begin{aligned}
    (h \ast_{\phi, f} k)(p)
    &= \int_{\mathcal{M}}h(q)k(f\circ\phi^{-1}(p) - f\circ\phi^{-1}(q)) dq\\
    &= \int_{\Omega} h\circ\phi(y) k(f(x) - f(y)) \, df(y) = \int_{\Omega} \tilde{h}(y)k(f(x) - f(y)) \, df(y)\\
    &= \int_{\Omega} \tilde{h}\circ f^{-1}(y')k(x' - y') \, dy' = \int_{\Omega} h^\#(y')k(x' - y') \, dy'\\
    &= (h^\# \ast k)(x').
    \label{eq:qcc2plain}
\end{aligned}
\end{equation}
\end{remark}

\begin{remark}\label{rm:any_para}
    Under the condition of Definition~\ref{def:qc_conv}, for any orientation preserving homeomorphism $\psi:\Omega \rightarrow\mathcal{M}$ and quasi-conformal mapping $\hat{f}:\Omega \to \Omega$. The convolution operator $\ast_{\psi, \hat{f}}$ is defined by
    \begin{equation}
        (h \ast_{\psi, \hat{f}} k)(p)
        = \int_{\mathcal{M}}h(\psi(\hat{f}^{-1}(y')))k(\hat{f}\circ\psi^{-1}(p) - y') dy'
    \end{equation}    
    By taking the quasi-conformal map $f = \hat{f}\circ\psi^{-1}\circ\phi$, we have
    \begin{equation}
    \begin{aligned}
        (h \ast_{\psi, \hat{f}} k)(p)
        &=\int_{\mathcal{M}}h(\psi\circ \hat{f}^{-1}(y'))k(\hat{f}\circ\psi^{-1}(p) - y') dy'\\
        &= \int_{\mathcal{M}}h(\psi\circ(f\circ\phi^{-1}\circ\psi)^{-1}(y'))k((f\circ\phi^{-1}\circ\psi)\circ\psi^{-1}(p) - y') dy'\\
        &=\int_{\mathcal{M}}h(\phi\circ f^{-1}(y'))k(f\circ\phi^{-1}(p) - y') dy'\\
        &=(h \ast_{\phi, f} k)(p).        
    \end{aligned}
    \end{equation}
    Therefore, $\ast_{\psi, \hat{f}} = \ast_{\phi, f}$, and hence $\ast_{\psi, \hat{f}}$ is a quasi-conformal convolution operator.
\end{remark}
}


\begin{theorem}
    Let $\mathcal{M} \subset \mathbb{R}^3$ be a 2-manifold and $*_\mathcal{M}$ be a parametrized manifold convolution. Then $*_\mathcal{M}$ is regular if and only if $*_\mathcal{M}$ is a quasi-conformal convolution.
\end{theorem}
\begin{proof}
    The mapping $\phi: \Omega \to \mathcal{M}$ is quasi-conformal if and only if it is an orientation-preserving homeomorphism, which is equivalent to saying that the parametrized manifold convolution $*_\phi$ is regular.
\end{proof}

The theorem above show that we can generalize any regular parametrized manifold convolution on simply connected surfaces into quasi-conformal convolution, through which a substantial subset of manifold convolution could be represented. {In addition, Remark \ref{rm:any_para} indicates that we can relax the conformal property of the initial surface parametrization $\phi$, which is convenient in practice.} As different quasi-conformal parameterizations would yield distinct convolution operators on that surface, one can possibly find the best convolution operator for a specific task by optimizing the corresponding quasi-conformal parameterization. {Also, as described in the definition, the QCC enables the Riemannian metric on the manifold to be defined by a quasi-conformal mapping, which provides a theoretical foundation for developing adjustable convolution on such surfaces.

Moreover, from a neural network design perspective, these theoretical foundations enable the introduction of a novel learnable convolution mechanism through a trainable quasi-conformal mapping. In this framework, the convolution window is no longer fixed but instead deforms adaptively in response to both the underlying geometry and data. This geometric adaptivity offers significantly greater flexibility for geometric deep learning. Importantly, such geometric flexibility is unique, distinct from merely increasing the parameter count.} 

In the next section, we will develop a deep neural network framework to learn the optimal quasi-conformal parameterization associated with the best convolution operator for a given task. Quasi-conformal mappings preserve local geometric structures while allowing controlled deformations. This characteristic ensures that the convolution operation aligns with the intrinsic geometry of the manifold, enabling more robust and effective feature extraction when implementing quasi-conformal convolution into deep neural networks.
\section{Quasi-Conformal Convolutional Neural Network on Riemann Surfaces}

\begin{figure}
    \centering
    \includegraphics[width=\linewidth]{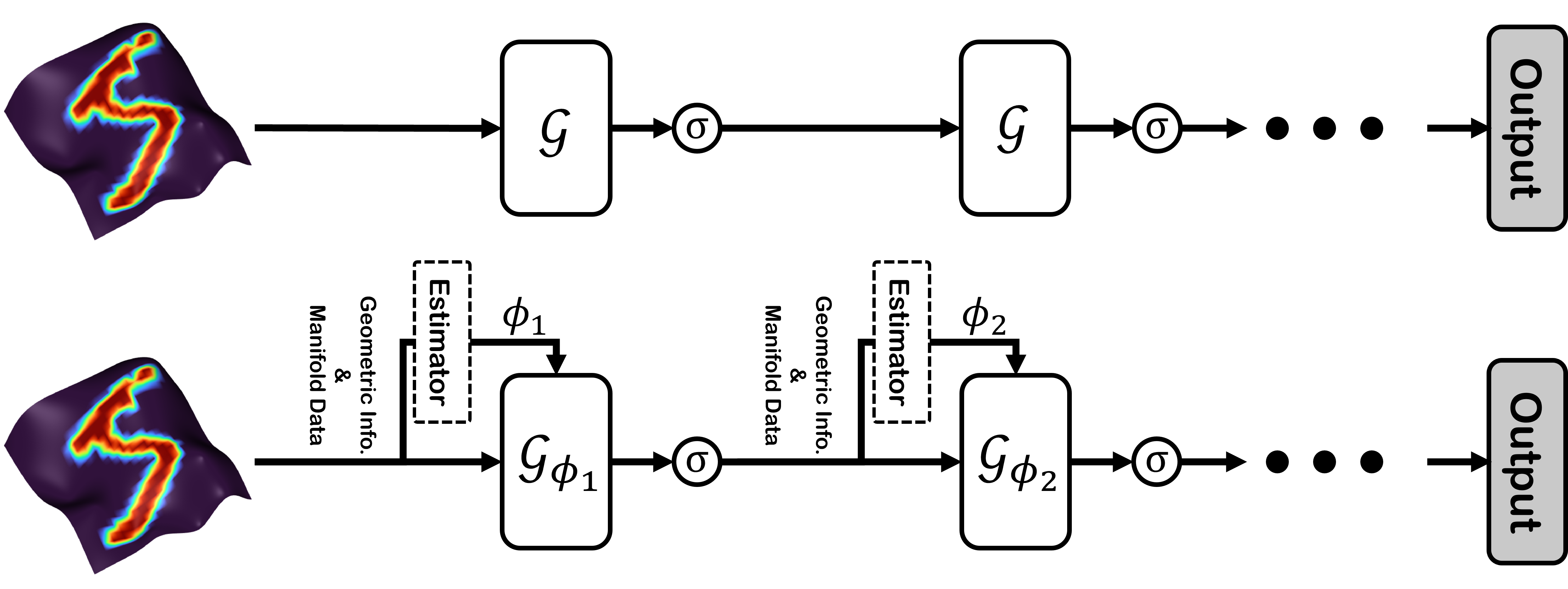}
    \caption{Illustration of a conventional convolutional neural network with predefined, untrainable convolution operations (top) compared to our proposed model (bottom) featuring learnable, data-driven convolution operations.}
    \label{fig:archi-compare}
\end{figure}

In this section, we aim to introduce the Quasi-Conformal Convolutional Neural Network (QCCNN) using quasi-conformal convolution as described in the previous section, which is capable of performing learning tasks on Riemann surfaces. We will present the architecture for a QC convolutional layer, detailing how it can be implemented in various settings.

The proposed QCCNN model incorporates learnable convolution operations that adapt dynamically to data through representing manifold convolution via quasi-conformal mappings. Unlike conventional models that rely on predefined and untrainable convolution operations, our architecture dynamically learns the convolution operation before entering each convolutional layer (Figure \ref{fig:archi-compare}). Specifically, at each layer $\mathcal{G}_i$, an auxiliary module, referred to as the Estimator, generates a mapping $\phi_i$ based on the feature map from the previous layer. This mapping, as described in Definition \ref{them:parameterconv}, serves as a parameterization that defines the learnable convolution operation for the current layer. The resulting feature map from $\mathcal{G}_i$ is then passed to subsequent layers, enabling a cascade of adaptive transformations.

\subsection{QC Convolutional Layer}

In this work, a key innovation lies in utilizing quasi-conformal mappings to learn an optimal parameterized manifold convolution. This section explores how QCC is integrated into deep learning methodologies and introduces the Quasi-conformal Convolutional Layer (QCC Layer), a pivotal component that extends standard convolution operations to 3D surfaces.

The QCC layer shares a similar purpose with conventional 2D convolutional layers: extracting high-dimensional feature maps from input data. However, unlike standard convolutional layers that operate on regular Euclidean grids, the QCC layer is designed to process manifold data $h$, where the convolution operations leverage a learnable, data-driven manifold convolution operation.

As shown in Figure \ref{fig:qcclayer}, the process begins by parameterizing the manifold data $ h $ through a parameterization function $ \phi$, which yields a pullback feature function $ \tilde{h} = h \circ \phi $, providing a common planar domain for subsequent computations. The pullback feature function $ \tilde{h} $ is then passed through the quasi-conformal mapping estimator (QCE), a specialized module that computes a quasi-conformal mapping $ f $. This mapping $ f $ adapts dynamically to the input data, reparametrizing $ \tilde{h} $ into $ h^\# $ as $ h^\# = \tilde{h} \circ f^{-1} $. At this stage, any conventional 2D convolution applied to $ h^\# $ is equivalent to performing a deformable convolution with $ \tilde{h} $, effectively bridging Euclidean and non-Euclidean operations.

\begin{figure}
    \centering
    \includegraphics[width=\textwidth]{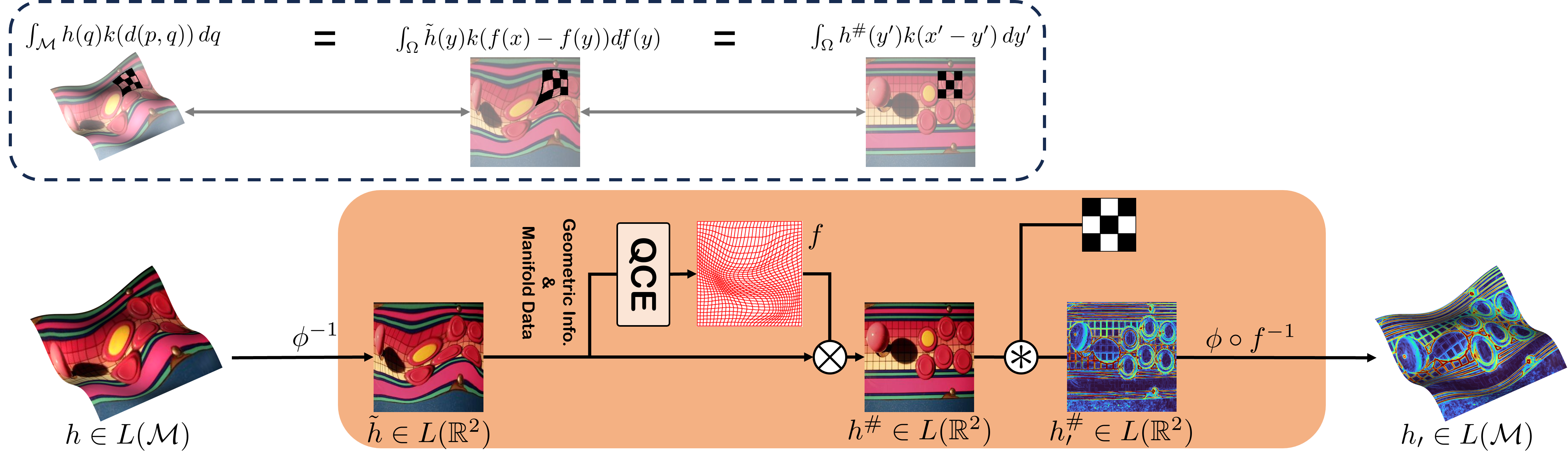}
    \caption{Illustration of the Quasi-conformal Convolutional layer: The process begins by parameterizing the manifold to establish a common planar domain for computations. A learnable, adaptive convolution is done through a data-driven quasi-conformal mapping. The dashed box highlights the equivalence between deformable convolution on the manifold, deformable convolution on the parameterized domain, and standard 2D convolution on the reparameterized domain via the quasi-conformal mapping.}
    \label{fig:qcclayer}
\end{figure}

Following this principle, the QCC layer allows the warped feature function $ h^\# $ to undergo convolution with a learnable kernel $ k $ in the planar domain. Mathematically, this operation represents a deformable convolution for $ \tilde{h} $, which can further be a deformable convolution with the original manifold data $ h $. Importantly, the kernel itself is deformable, with both its weights and shape dynamically adjusted based on the input data. This adaptability is achieved through the proposed QCE module, which enables the QCC layer to learn not only the optimal weights but also the geometric configuration of the kernels during training.

The output of this convolution operation, $ h^\#_\prime $, is then mapped back to the original manifold structure. This involves reversing the quasi-conformal deformation $ f $ and the parameterization function $ \phi $. This step ensures that the extracted features are mapped back onto the manifold, completing the process.

Importantly, the mapping $ f $ produced by the QCE, which we generally use a UNet indicated as in Figure \ref{fig:QCE}, should satisfy the property of being a quasi-conformal mapping, which requires its associated Beltrami coefficient to remain strictly less than 1. To enforce this condition, additional regularization terms are introduced:  
\begin{equation}\label{eq:bclosses}
\begin{aligned}
    \mathcal{L}_{\text{bc}} &= ||\mu(f)||_2, \\
    \mathcal{L}_{\text{lap}} &= ||\Delta f||_2,
\end{aligned}
\end{equation}
where $ \mu(f) $ represents the Beltrami coefficient of the mapping $ f $, computed via a Finite Difference Method implementation of Equation \eqref{eq:beleq}. 
To drive the Beltrami coefficient in our model to maintain a supremum norm strictly less than 1, we aim to control the size of this coefficient by minimizing a properly weighted $ \mathcal{L}_{\text{bc}} $. Empirical results demonstrate that our approach effectively reduces the Beltrami coefficient below this threshold, thereby encouraging the mapping $f$ to satisfy the quasi-conformal property. The term $ \Delta f $ denotes the Laplacian of the mapping $ f $, and the regularization $ \mathcal{L}_{\text{lap}} $ is included to promote smoothness in the mapping. Together, these regularizations ensure that $ f $ adheres to the quasi-conformal constraints while maintaining desirable geometric properties, such as continuity and smoothness.

The QCC layer's unique capability to integrate quasi-conformal mappings into the convolution process distinguishes it from traditional manifold learning approaches. By learning and dynamically adjusting both the kernel shape and kernel weightings, the QCC layer offers a robust framework for handling non-Euclidean domains. This adaptability leaves convolution operations dynamic and learnable, resolving a significant challenge in manifold convolution and enhancing the effectiveness of learning on manifold data.

\begin{figure}
    \centering
    \includegraphics[width=0.85\textwidth]{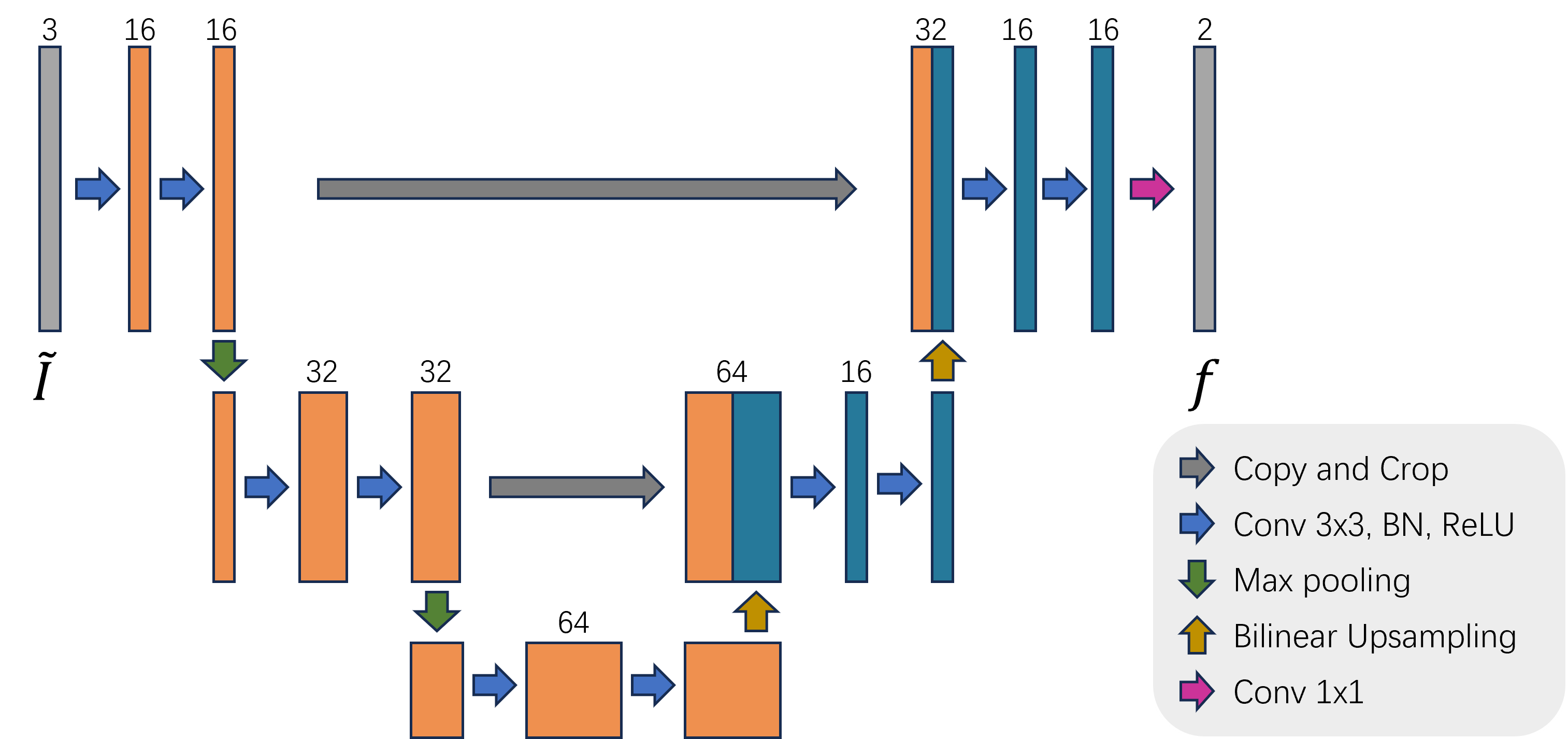}
    \caption{The architecture of the QC mapping estimator.}
    \label{fig:QCE}
\end{figure}

\subsection{QC Convolutional Neural Network}
\label{QCCNN}


Then, we delve into the construction of a Quasi-conformal Convolutional Neural Network (QCCNN), which integrates the quasi-conformal convolutional layer into a fully functional deep learning framework.

For simplicity in our discussion, we outline the design of a QCCNN where each layer has a single channel. Extending the architecture to multiple channels is straightforward. The network begins with an input manifold data $ h $, and subsequent feature maps $ h_i $ are computed iteratively using the QCC operation with kernel functions $k_i$ in each layer. Mathematically, the forward propagation of features through the network of $n$ layers can be expressed as follows:
\begin{equation}
\begin{aligned}
    h_1\, &= h,\\
    h_2\, &= \sigma(h_1 \ast_{\phi, f_1} k_1),\\
    h_3\, &= \sigma(h_2 \ast_{\phi, f_2} k_2),\\
    &\qquad \vdots\\    
    h_{out} &= \sigma(h_{n} \,\ast_{\phi, f_{n}} k_{n}),
\label{eq:formulationQCCNNmulti}
\end{aligned}
\end{equation}
where $ \sigma $ represents a non-linear activation function, $ k_i $ are the learnable convolution kernels, and $ f_i $ are the quasi-conformal mappings used to define the QCC at layer $ i $. The operator $ \ast_{\phi, f_i} $ denotes the QCC operation defined earlier.

A crucial aspect of the QCCNN is the role of the mapping $ f_i $, which determines the convolution operation on the manifold and governs the deformation of the convolution windows. While $ f_i $ can vary across layers, allowing each layer to learn a layer-specific mapping $ f_i $, a unified mapping $ f $, shared across all layers ($ f_i = f $ for every layer $ i $), is also a viable and meaningful choice. 
Since the first estimator captures most geometric features of the surfaces and outputs an optimal quasi-conformal mapping $f$ for the parametrized convolution, it is reasonable that such mapping $f$ is also effective for the latter layers. Moreover, using independent estimators to generate quasi-conformal mappings at each layer would significantly increase memory and computational costs for the entire network. Therefore, it is more efficient to estimate the quasi-conformal mapping at the beginning of the model and apply this mapping consistently across the QCC layers in the subsequent hidden layers to avoid unnecessary computations.

Considering that a consistent QCC operation across layers is often sufficient, we discuss using a uniform mapping $ f $ for each layer. This leads the QCCNN to the following formulation:
\begin{equation}
\begin{aligned}
    h_1\, &= h, \\
    h_2\, &= \sigma(h_1 \ast_{\phi, f} k_1), \\
    h_3\, &= \sigma(h_2 \ast_{\phi, f} k_2), \\
    &\qquad \vdots\\    
    h_{out} &= \sigma(h_n \,\ast_{\phi, f} k_n).
\label{eq:formulationQCCNNsingle}
\end{aligned}
\end{equation}

According to the formulation in Equation \eqref{eq:qcc2plain}, we can rewrite Equation \eqref{eq:formulationQCCNNsingle} above as
\begin{equation}
\begin{array}{rll}
    h_1\, &= h &= h_1^\# \circ f \circ \phi^{-1}, \\
    h_2\, &= \sigma(h_1 \ast_{\phi, f} k_1) &= \sigma(h_1^\# \ast k_1) \circ f \circ \phi^{-1}, \\
    h_3\, &= \sigma(h_2 \ast_{\phi, f} k_2) &= \sigma(h_2^\# \ast k_2) \circ f \circ \phi^{-1}, \\
    &\qquad \vdots \\
    h_{out} &= \sigma(h_n \ast_{\phi, f} k_n) &= \sigma(h_n^\# \ast k_n) \circ f \circ \phi^{-1}.
\end{array}
\end{equation}

From the observation on the above formulation, if one single quasi-conformal mapping is employed for each QCC layer, performing QC convolution between manifold feature function $ h_i $ and a sequence of kernels is equivalent to performing a sequence of kernels on the transformed pullback function $ h_i^\# $ in the parametrized domain. Thus, given the input data $ h $, performing 2D CNN on the transformed pullback function $ h^\# $ is equivalent to performing a QCCNN, which has the same architecture as the 2D CNN with each convolutional layer replaced by a QC convolutional layer under $ \phi $ and $ f $, on the manifold function $ h $.

In other words, given an input feature function $ h $ in a manifold and a QCE $ \mathcal{M}_{\text{QCE}} $, for any 2D CNN $ \mathcal{N} $, we could obtain a QCCNN $ \mathcal{N}_{f}(h) = \mathcal{N}(h^\#) $, where $ h^\# = h \circ \phi \circ f^{-1} $ and $ f = \mathcal{M}_{\text{QCE}}(h \circ \phi) $. The QCCNN $ \mathcal{N}_f $ should have the same architecture as the 2D CNN by replacing all plain 2D convolutions ($ \ast $) in $ \mathcal{N} $ with $ \ast_{\phi, f} $. This discussion demonstrates that to perform such a QCCNN, we can estimate one single QC mapping $ f $ to obtain $ h^\# $ and directly apply a plain 2D CNN on $ h^\# $, as shown in Figure \ref{fig:qccimplementation}. This is equivalent to directly applying a QCCNN on the manifold function $ h $ with a learned QC mapping $ f $. 

In the quasi-conformal convolution framework, there is no direct pooling operation on the manifold itself. However, pooling, traditionally used to reduce the spatial dimensions of feature maps and adjust the domain's scale for learning, can be effectively implemented using standard 2D pooling operations in the planar parameterization domain. This approach is particularly advantageous in QCC, as the parameterization mapping $f$ can represent any regular manifold convolutions, thereby accommodating deformable pooling operations in a consistent manner.

\begin{figure}
    \centering
    \includegraphics[width=0.9\textwidth]{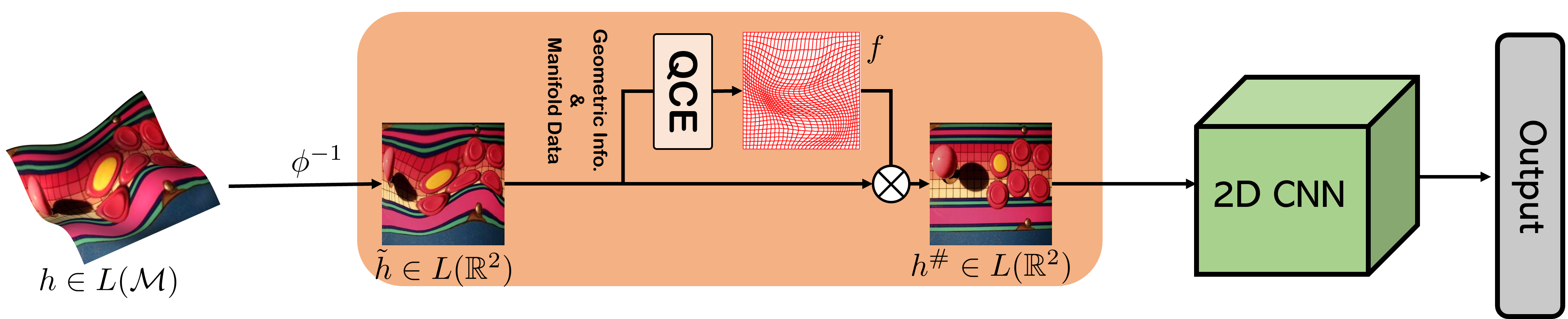}
    \caption{QCC implementation with a single learned quasi-conformal (QC) mapping $f$: The QCCNN is constructed by estimating a single QC mapping $f$, transforming the parameterized input manifold feature into a reparameterized domain, and then applying a standard 2D CNN.}
    \label{fig:qccimplementation}
\end{figure}


{

Thus, the training of the QCC Neural Network can be formulated as an optimization problem. Specifically, we define:
\begin{equation}
\begin{aligned}
    \tilde{h}\;\;\, &= h \circ \phi\\
     f\;\;\,  &= \mathcal{M}_\text{QCE}(\tilde{h};\boldsymbol{\theta}_\mathcal{M})\\
     h^{\#} &= \tilde{h} \circ f^{-1}\\
     y\;\;\,  &= \mathcal{N}(h^{\#};\boldsymbol{\theta}_\mathcal{N})
\end{aligned}    
\end{equation}
where $\phi$ is a precomputed mapping, and $\boldsymbol{\theta}_\mathcal{M}$ and $\boldsymbol{\theta}_\mathcal{N}$ denote the trainable parameters of the QCE and downstream networks, respectively.

To solve a specific task, we define an associated loss function $\mathcal{L}_\text{task}$. For instance, this may be the cross-entropy loss for classification or the mean squared error (MSE) loss for segmentation. Combined with the regularization terms from \eqref{eq:bclosses}, we arrive at the unified objective:
\begin{equation}
    \mathcal{L}(h;\boldsymbol{\theta}_\mathcal{M},\boldsymbol{\theta}_\mathcal{N}) = \lambda_\text{task} \mathcal{L}_\text{task}(h;\boldsymbol{\theta}_\mathcal{M},\boldsymbol{\theta}_\mathcal{N}) + \lambda_\text{bc} \mathcal{L}_\text{bc}(h;\boldsymbol{\theta}_\mathcal{M}) + \lambda_\text{lap} \mathcal{L}_\text{lap}(h;\boldsymbol{\theta}_\mathcal{M}),
\end{equation}
where $\lambda_\text{task}$, $\lambda_\text{bc}$, and $\lambda_\text{lap}$ are weighting coefficients that balance the contributions of each loss term.

Note that $h$ represents the input data and is not subject to optimization. Therefore, to achieve the desired task performance, we optimize over the network parameters, leading to the following minimization problem:
\begin{equation}
\min\limits_{\boldsymbol{\theta}_\mathcal{M},\boldsymbol{\theta}_\mathcal{N}}\mathcal{L}(h;\boldsymbol{\theta}_\mathcal{M},\boldsymbol{\theta}_\mathcal{N}),
\end{equation}

To solve this optimization problem, we apply gradient descent methods to update $\boldsymbol{\theta}_\mathcal{M}$ and $\boldsymbol{\theta}_\mathcal{N}$ using two separate optimizers:
\begin{equation}
\begin{aligned}
\boldsymbol{\theta}_\mathcal{M} &\leftarrow \boldsymbol{\theta}_\mathcal{M} - \tau_\mathcal{M} \nabla_{\boldsymbol{\theta}_\mathcal{M}} \mathcal{L}(\boldsymbol{\theta}_\mathcal{M} , \boldsymbol{\theta}_\mathcal{N}),\\
\boldsymbol{\theta}_\mathcal{N\,} &\leftarrow \boldsymbol{\theta}_\mathcal{N\,} - \tau_\mathcal{N\,} \nabla_{\boldsymbol{\theta}_\mathcal{N\,}} \mathcal{L}(\boldsymbol{\theta}_\mathcal{M} , \boldsymbol{\theta}_\mathcal{N}).
\end{aligned}
\label{eq:optimizer3D}
\end{equation}
where $\nabla{\boldsymbol{\theta}_\mathcal{M}}$ and $\nabla{\boldsymbol{\theta}_\mathcal{N}}$ denote the gradients with respect to the corresponding parameters. The step sizes $\tau_\mathcal{M}$ and $\tau_\mathcal{N}$ are adaptively adjusted using methods such as those described in \cite{amari1993backpropagation,KingBa15}. Alternating optimization between the two components is adopted; the specific configuration will be provided in the experimental section.
}
\section{Experiments}
This section provides the details of the implementation of our experiments and discusses their results. We evaluate the performance of our proposed model in comparison with other methods, such as PTCNet and GCNN. The experiments encompass a variety of tasks, including image classification on manifolds, craniofacial analysis, and segmentation of facial lesions. Additionally, to analyze the impact of different network architectures and hyperparameter configurations, we conduct ablation studies to test our model under various setups.

The experimental setting is described in detail below.

\noindent\textbf{Computational Resources and Parameters}
The training parameters are set as follows: \(\lambda_\text{bc} = 0.01\), \(\lambda_\text{lap} = 0.0001\), and the learning rate \(lr = 1.0 \times 10^{-5}\). The training is conducted on a CentOS 8.1 central cluster computing node equipped with two Intel Xeon Gold 5220R 24-core CPUs and two NVIDIA V100 Tensor Core GPUs.

\noindent\textbf{Conformal Convolutional Neural Network} The proposed model's ability to learn an adaptive convolution through the mapping $f$ in Equation \eqref{eq:formulationQCCNNsingle} is a key advantage. To assess the benefits of these learnable features, we also evaluate a baseline approach, the Conformal Convolutional Neural Network (CCNN), which relies solely on a non-trainable parametrized manifold convolution operation from conformal parameterization as introduced by Definition \ref{def:conformalconv}. By comparing results from CCNN and our model across various experiments, we reveal the significant advantages brought by the learned optimal parametrized manifold convolution.
\subsection{Image Classification on Riemann Surfaces}
\newcolumntype{A}{>{\centering\arraybackslash}p{0.18\columnwidth}}
\begin{table}
\centering
\begin{tabular}{c|AAAA}
\toprule
Method  & Single-Simple & Multi-Simple  & Single-Complex& Multi-Complex   \\\hline
CCNN    & 97.08         & 96.77         & 91.12         & 86.57        \\ 
GCNN    & 96.22         & 96.97         & 90.34         & 88.84         \\ 
PTCNet  & 97.96         & 97.32         & 92.86         & 86.06         \\ 
QCCNN   & \bf 99.11         & \bf 98.60         & \bf 97.45         & \bf 94.16         \\\hline
\bottomrule
\end{tabular}
\caption{Quantitative comparison of manifold image classification results.}
\label{tb:classification}
\end{table}

In this section, we evaluate the performance of our proposed method on image classification tasks involving MNIST digits mapped onto Riemann surfaces. The MNIST dataset consists of 60,000 handwritten digits spanning ten classes (0-9). To assess the effectiveness of QCCNN, we conduct experiments on Riemann surfaces exhibiting both simple and complex geometric structures.
Simple surfaces are relatively flat, making the classification tasks of MNIST data on them easier. In contrast, complex surfaces are irregular and exhibit significant geometric fluctuations, posing substantial challenges for the classification tasks. Examples of these surfaces with MNIST digits on them are shown in Figure~\ref{fig:mnist_surfaces}. 

For comparison, we benchmark our results against GCNN and PTCNet in Table \ref{tb:classification}. All networks are designed to include only one convolutional layer and one fully connected layer, consistent with the approach in \cite{schonsheck2022parallel}, to ensure a fair evaluation. Our network structure is illustrated in Figure \ref{fig:diagram1}. The network takes the geometric information of the surface, including Gaussian curvature and mean curvature, together with the texture information as the input to output the optimal quasi-conformal map associated with the best parameterized manifold convolution for the classification task.

In practice, we first compute the curvatures of the surface \cite{Dastan2025}, and then map the surface to a 2D domain \cite{meng2016conformal} as the input of both the estimator and the classifier. In the training stage, the batch size is fixed at 64 for all models. For QCCNN, we first train the classifier independently for 500 epochs, followed by alternating training of the classifier and the estimator for another 500 epochs, with the training alternating every 100 epochs. The loss function is given by

\begin{equation}\label{loss_total}
    \mathcal{L}_\text{total} = \lambda_\text{entropy} \mathcal{L}_\text{entropy} + \lambda_\text{bc} \mathcal{L}_\text{bc} + \lambda_\text{lap} \mathcal{L}_\text{lap},
\end{equation}
where $\mathcal{L}_\text{entropy}$ is the cross entropy loss and $\mathcal{L}_\text{bc}$ and $\mathcal{L}_\text{lap}$ are as defined in Equation (\ref{eq:bclosses}). $\lambda_\text{entropy}=1$, $\lambda_\text{bc}=0.01$ and $\lambda_\text{lap}=0.0001$ are positive weighting coefficients.

\begin{figure}
    \centering
    \includegraphics[width=\linewidth]{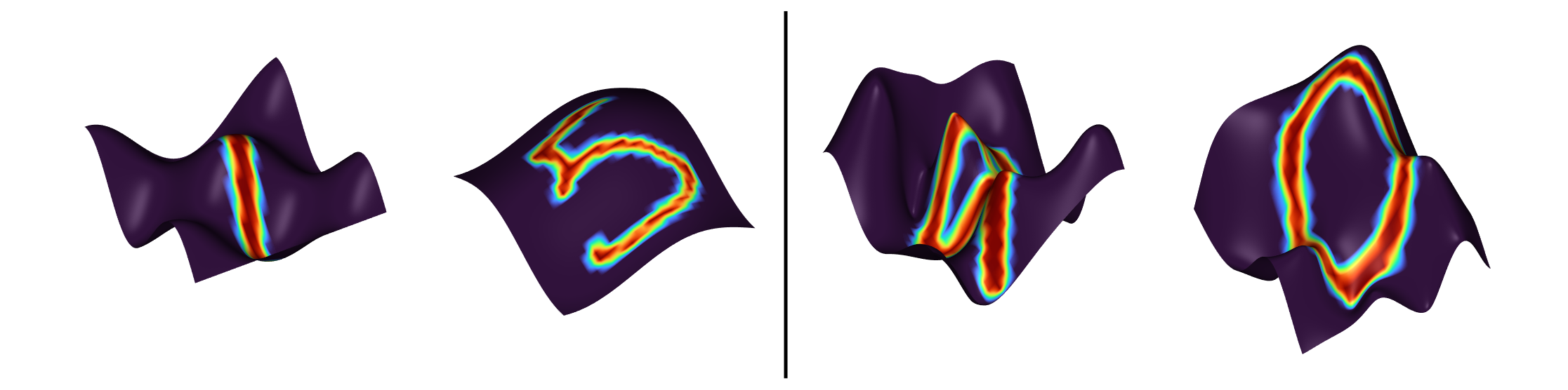}
    \caption{MNIST images printed on simple surfaces (left) and complex surfaces (right)}
    \label{fig:mnist_surfaces}
\end{figure}

\begin{figure}
    \centering
    \includegraphics[height=50mm]{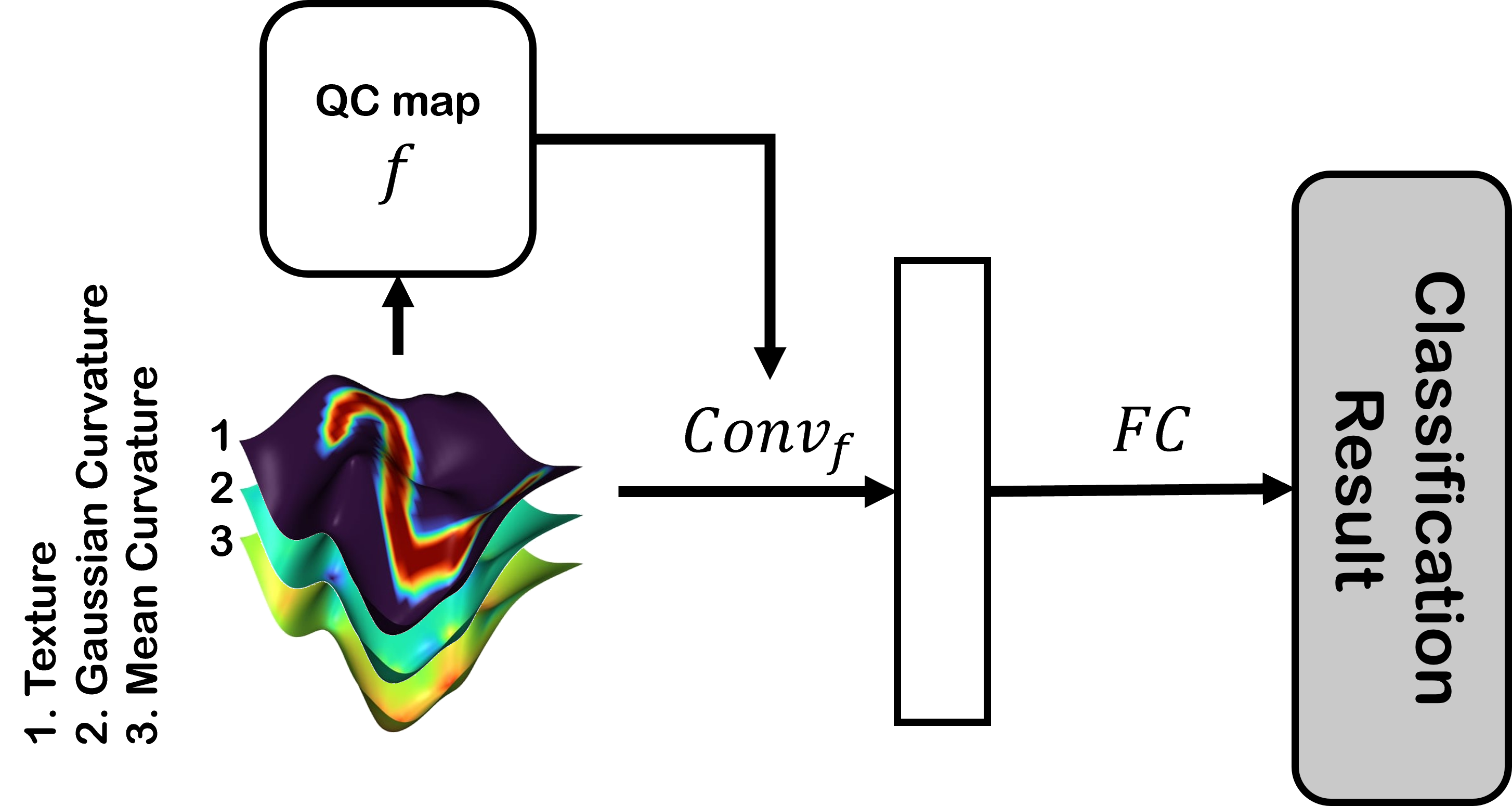}
    \caption{Illustration diagram of QCCNN for classification of MNIST on Riemann surfaces.}
    \label{fig:diagram1}
\end{figure}

\subsubsection{Classification on Simple Manifold}
For simple surfaces, we generate them using basic trigonometric functions to produce $z$-coordinate in a range of $[-0.2,0.2]$ given $x \in [0,1]$, $y \in [0,1]$, which is similar to the approach in PTCNet experiments. 

Two distinct setups are considered for simple manifolds: single-simple and multi-simple. In the single-simple setup, we use a single surface for both training and testing. In contrast, the multi-simple setup involves training on four surfaces and testing on a different, unseen surface. These setups are inspired by \cite{schonsheck2022parallel}, and the results for GCNN and PTCNet are taken from the same reference. As shown in Table \ref{tb:classification}, our model achieves the highest accuracy, outperforming both PTCNet and GCNN.

\subsubsection{Classification on Complex Manifold}
For complex surfaces, we first randomly sample 24 points within $[-0.2, 1.2]^2 \subset \mathbb{R}^2$, which contains the unit square $[0,1]^2$. Then, each point is assigned with a random height value in $[-0.4,0.4]$ as the $z$-coordinate. The surface $S$ is then constructed through biharmonic spline interpolation among these points and the subset $S'=\{(x,y,z) \in S : (x,y) \in [0,1]^2\} \subset S$ is the desired surface for the learning task. Such cropping ensures that any over-fluctuating boundary can be removed, but $S$ may still contain some points with $z$-coordinate outside $[-0.4,0.4]$. Due to the significantly larger range and variance of $z$-coordinates compared to simple surfaces, these complex surfaces are much more irregular and challenging, posing a significant test for surface learning models.

As with the experiments for simple manifolds, two setups are considered: single-complex and multi-complex. In the single-complex setup, one surface is used for both training and testing. In the multi-complex setup, 50 surfaces are used for training, and another set of 50 unseen surfaces is used for testing.

We follow the same training protocol for QCCNN as used in the simple manifold experiments. As shown in Table \ref{tb:classification}, our model achieves the highest accuracy, surpassing both GCNN and PTCNet.

\subsection{Craniofacial Analysis}
\newcolumntype{C}{>{\centering\arraybackslash}p{0.12\columnwidth}}
\begin{table}[h!]
\centering
\begin{tabular}{c|CCCCC}
\toprule
Method  & Accuracy  & Sensitivity & Precision    & N.P.V.   & Specificity   \\\hline
CCNN    & 91.33     & 90.00     & 92.47     & 90.26         & 92.67 \\
GCNN    & 89.33     & 91.33     & 87.82     & 90.97         & 87.33\\ 
PTCNet   & 91.00     & 92.00     & 90.20     & 91.84         & 90.00\\ 
QCCNN   & \bf 97.00     & \bf 96.67     & \bf 97.32     & \bf 96.69         & \bf 97.33
\\\hline
\bottomrule
\end{tabular}
\caption{Quantitative comparison of craniofacial analysis results.}
\label{tb:Craniofacial}
\end{table}

\begin{figure}
    \centering
    \includegraphics[height=50mm]{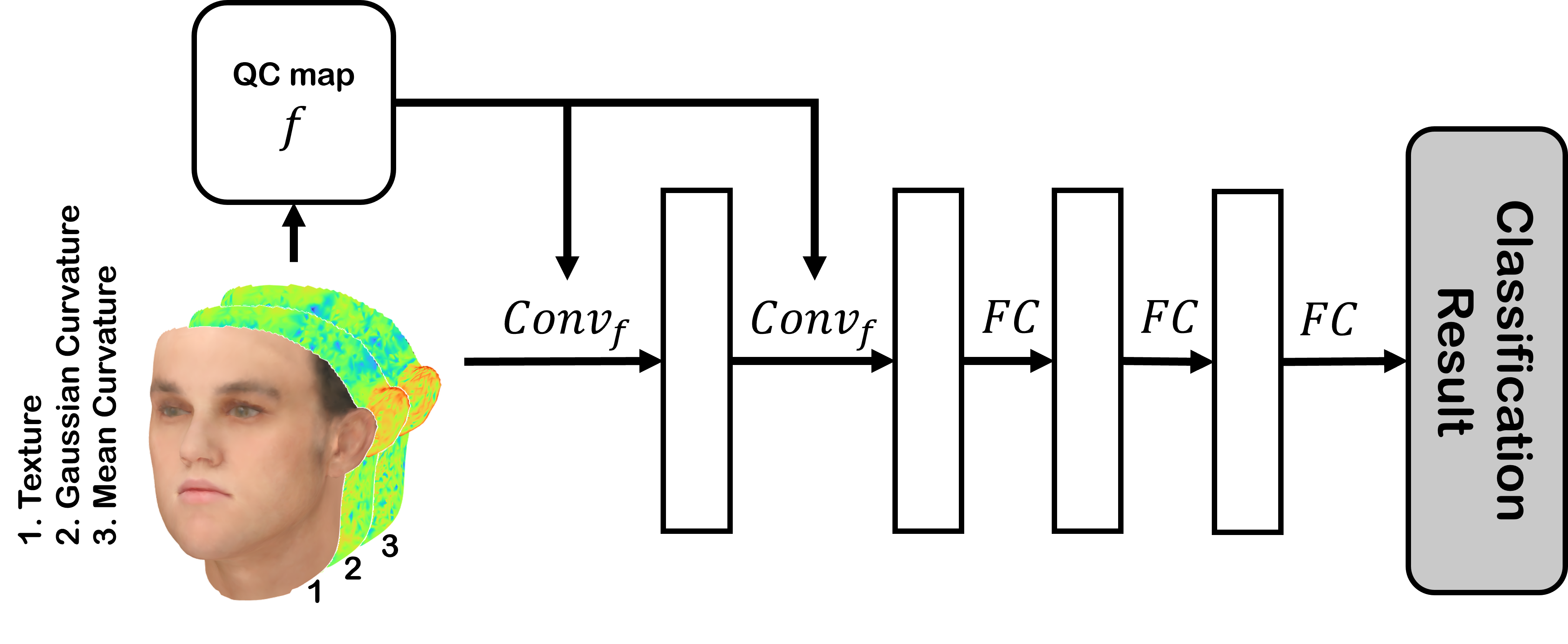}
    \caption{Illustration diagram of QCCNN for craniofacial analysis.}
    \label{fig:diagram2}
\end{figure}

\begin{figure}
    \centering
    \includegraphics[width=\linewidth]{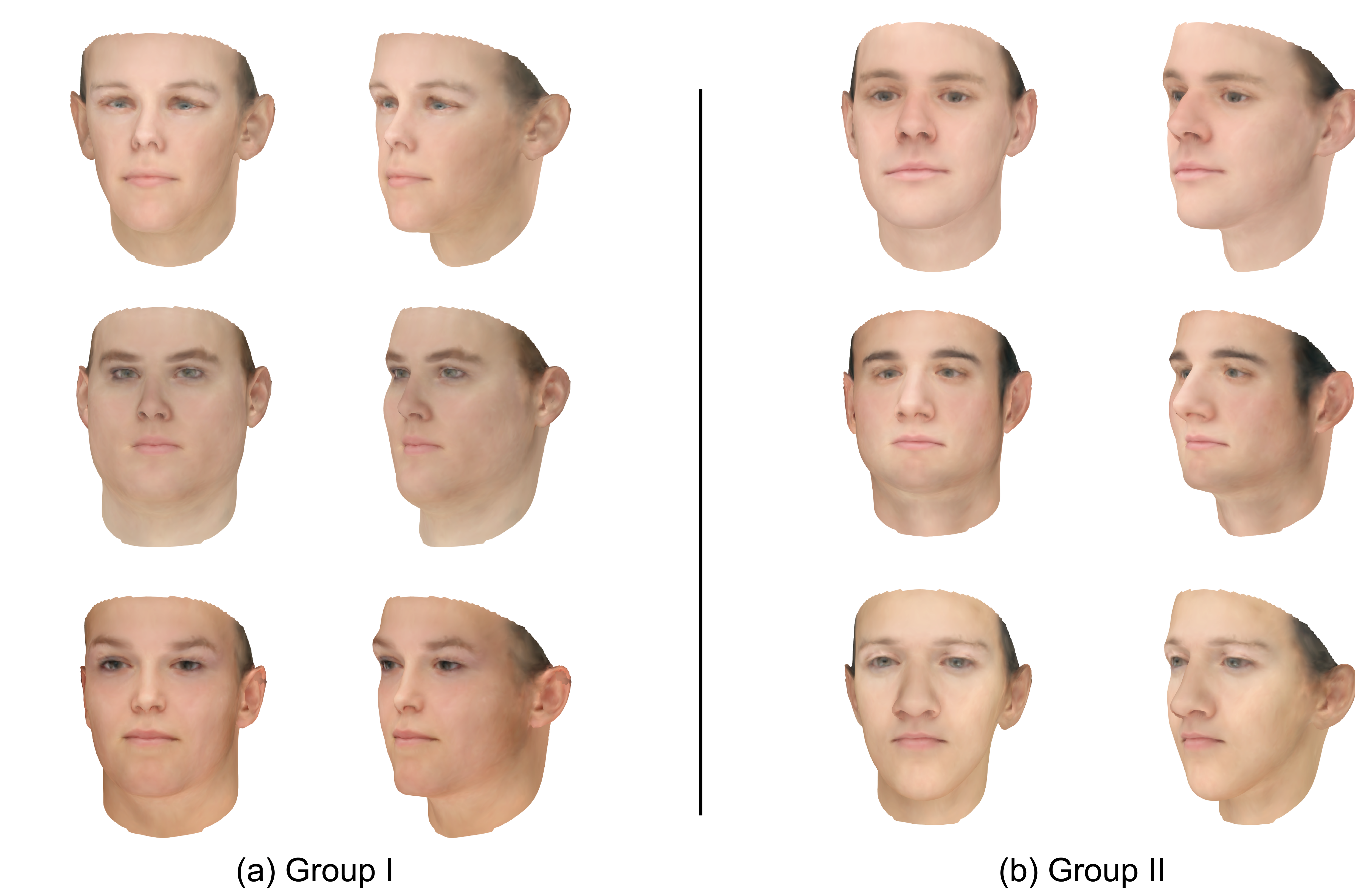}
    \caption{Craniofacial analysis examples for faces with different nasal structures.}
    \label{fig:faceosa}
\end{figure}
Craniofacial analysis aims to study the structure and relationships of the bones and tissues of the skull and face, with important applications in orthodontics, forensic science, reconstructive surgery, and medical image analysis. For instance, craniofacial analysis has been used to study nasal structure to evaluate facial asymmetries, plan surgical interventions, and serve as an excellent clinical parameter for various orthodontic treatments \cite{nehra2009nasal,burstone1958integumental} and disease diagnoses \cite{matthews2022static}.

A dataset of 2D human facial surfaces is used, consisting of two groups with 500 subjects each. Each group exhibits distinct nasal structural patterns that may include variations in size, symmetry, and other characteristics. Among these samples, 70\% are used for training, while the remaining 30\% are reserved for testing. Representative examples from the two groups, showing different nasal structural patterns, are presented in Figure \ref{fig:faceosa}.


During the parameterization process, each 3D facial surface is mapped to a 2D image of size $128 \times 128$. For the classification task, we use a neural network architecture comprising two convolutional layers and three fully connected layers, as shown in Figure \ref{fig:diagram2}. Gaussian curvature and mean curvature, which together provide a comprehensive set of geometric descriptors of the Riemann surface, along with texture information, are input into the network. The network outputs the quasiconformal mapping associated with the optimal QCC and the classification result. A batch size of 50 is used for training. For QCCNN, the classifier and estimator are trained alternately for 100 epochs each, resulting in a total of 1000 epochs. As in the previous experiment, Equation \ref{loss_total} is used as the loss function for this classification task.

The results of this study, as summarized in Table \ref{tb:Craniofacial}, demonstrate that our QCCNN outperforms Parallel Transport Convolutional Networks (PTCNet), Geodesic Convolutional Neural Networks (GCNN), and Conformal Convolutional Neural Networks (CCNN) in classifying nasal structural differences on Riemann surfaces. This improvement can be attributed to the flexibility of the QCCNN framework, which dynamically adapts convolution operations based on both texture and geometric information through the use of quasi-conformal mappings. Unlike traditional networks and CCNN, which rely on predefined convolution operations, QCCNN enables the convolution process to account for subtle variations in different surface geometries. This adaptability enhances the ability of the model to extract discriminative features, resulting in superior performance.

\subsection{Facial Lesion Segmentation}
\label{sec:self_ablation}

\newcolumntype{B}{>{\centering\arraybackslash}p{0.13\columnwidth}}
\begin{table}
\centering
\begin{tabular}{c|BBBBB}
\toprule
Method  & MSE$\times 10^{-2}$   & Dice  & IoU   & Sensitivity   & HD95  \\\hline
CCNN     & 2.469                 & 83.06 & 75.98 & 85.74         & 18.12\\ 
GCNN    & 3.124                 & 81.18 & 72.52 & 83.30         & 12.05\\ 
QCCNN   & \bf 0.878                 & \bf 88.62 & \bf 83.15 & \bf 91.52         & \bf 10.45\\\hline
\bottomrule
\end{tabular}
\caption{Quantitative comparison of facial lesion segmentation.}
\label{tb:leision}
\end{table}

\begin{figure}
    \centering
    \includegraphics[height=50mm]{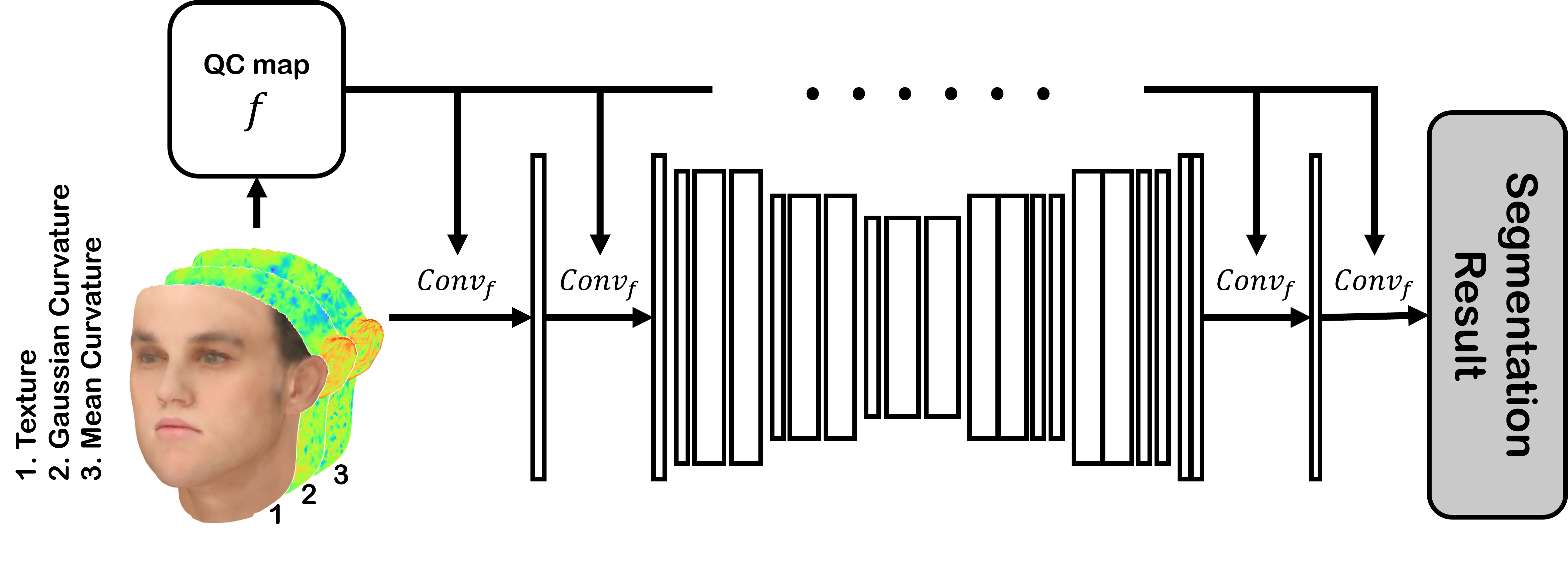}
    \caption{Illustration diagram of QCCNN for face lesion experiment.}
    \label{fig:diagram3}
\end{figure}

\begin{figure}
    \centering
    \includegraphics[width=\linewidth]{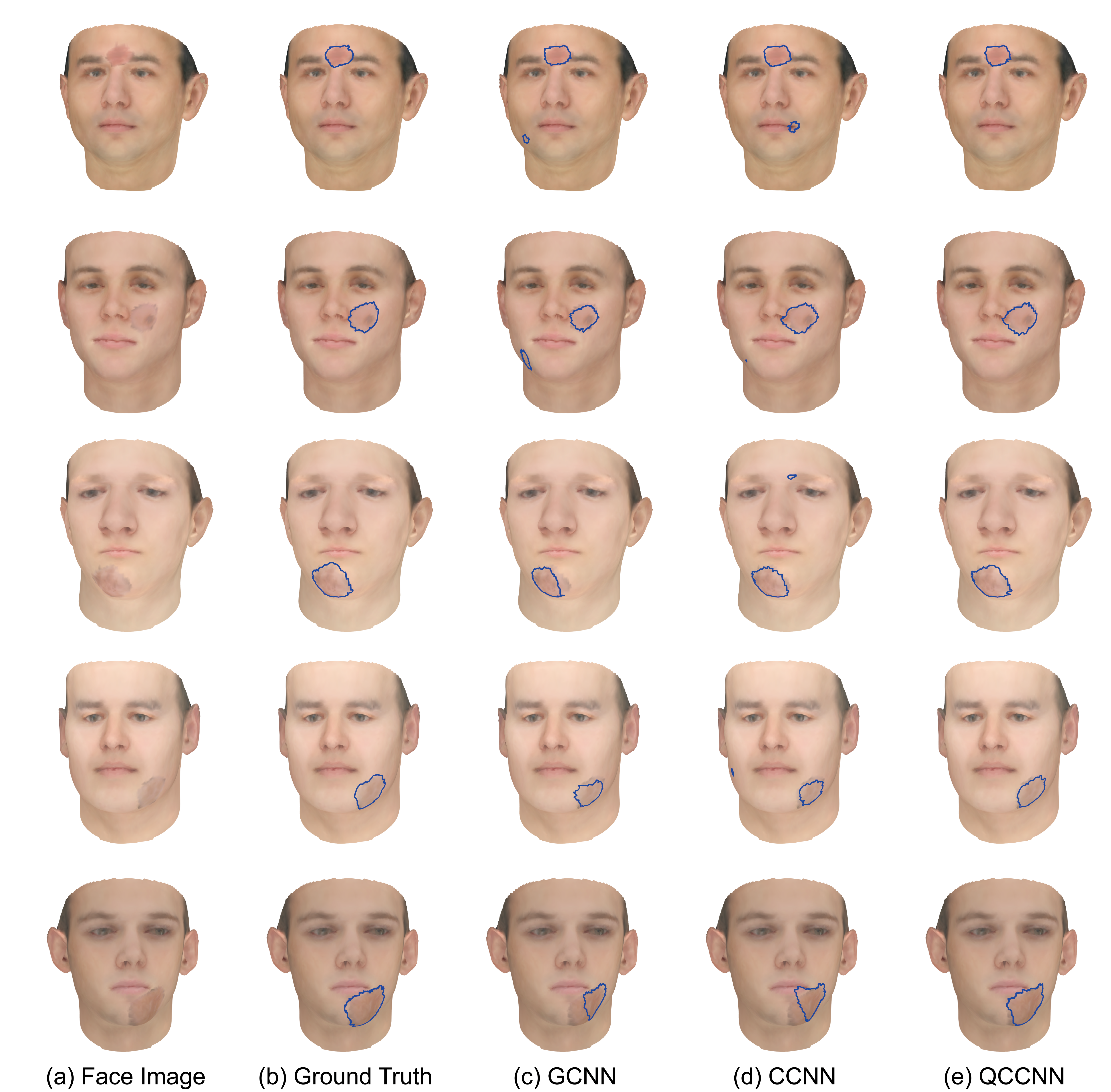}
    \caption{Qualitative comparison of facial lesion segmentation results.}
    \label{fig:facelesion}
\end{figure}

Segmentation on Riemann surfaces has a broad range of applications, including skin lesion segmentation, facial feature extraction, and medical image analysis. In this experiment, we evaluate our proposed framework for the segmentation task on Riemann surfaces. Specifically, we apply our framework to segment skin lesions on human faces. The experiments are conducted on a dataset of 4000 human faces with skin lesions located in various regions. Each face is accompanied by a ground truth segmentation mask. Of the 4000 labeled samples, 3200 are used for training, while the remaining 800 are reserved for testing the supervised surface segmentation deep neural network. 

To evaluate the necessity and advantages of incorporating quasi-conformal mappings within our QCC framework, we compare its performance against the Geodesic Convolutional Neural Network (GCNN) and the Conformal Convolutional Neural Network (CCNN). The results highlight the importance of learnable convolutions enabled by QCCNN and demonstrate its significant benefits over alternative methods.


During the parameterization stage, the surfaces are conformally mapped into images of size $128 \times 128$. The training process is conducted with a batch size of 50. The architecture of the QCCNN for the segmentation task on Riemann surfaces is illustrated in Figure \ref{fig:diagram3}. The network incorporates QC estimators to compute the quasiconformal mappings associated with the optimal QC convolution for each input surface. A UNet architecture with three downsampling layers is employed for segmentation. The inputs to the network include texture information, Gaussian curvature, and mean curvature. The loss function is given by

\begin{equation}\label{loss_seg}
    \mathcal{L}_\text{seg} = \lambda_\text{mse} \mathcal{L}_\text{mse} + \lambda_\text{bc} \mathcal{L}_\text{bc} + \lambda_\text{lap} \mathcal{L}_\text{lap},
\end{equation}
where $\mathcal{L}_\text{mse}$ is the mean square loss and $\mathcal{L}_\text{bc}$ and $\mathcal{L}_\text{lap}$ are as defined in Equation \ref{eq:bclosses}. $\lambda_\text{mse}=100$, $\lambda_\text{bc}=0.01$ and $\lambda_\text{lap}=0.0001$ are positive weighting coefficients.

Qualitative results are shown in Figure \ref{fig:facelesion}. In the first two rows, false-positive regions appear in both CCNN and GCNN results due to closely matching texture patterns and intrinsic geometric features in non-lesion areas, leading to segmentation inaccuracies. In the third to fifth rows, GCNN fails to segment the entire lesion region. Similar failures are observed in the results from CCNN, emphasizing the advancements introduced by the proposed QCCNN method, particularly its use of adaptive and learnable convolutions.

Quantitative results are provided in Table \ref{tb:leision}. The proposed method outperforms the competing approaches in both standard metrics, such as the Dice coefficient, and geometric metrics, such as HD95. These results demonstrate the robustness of the QCCNN in accurately segmenting lesion regions by effectively leveraging adaptive, data-aware convolution kernels and kernel shapes.

\subsection{Self Ablation}

In this section, we perform self-ablation studies to analyze the impact of various factors in the proposed Quasi-conformal Convolutional Network (QCCNN) model on overall performance.

\newcolumntype{D}{>{\centering\arraybackslash}p{0.12\columnwidth}}
\begin{table}
\centering
\begin{tabular}{c|DDDDD}
\toprule
Down Ratio& Accuracy  & Trainable Params. & For./Back. Pass Size  & Params. Size  & Total Size\\\hline
$\times 2^{-1}$ & 93.83     & 23,634            & 0.93 MB               & 0.09 MB       & 1.03 MB\\ 
$\times 2^{-2}$ & 97.45     & 99,858            & 1.02 MB               & 0.38 MB       & 1.41 MB\\ 
$\times 2^{-3}$ & \bf 97.55 & 404,370           & 1.02 MB               & 1.54 MB       & 2.57 MB\\ 
$\times 2^{-4}$ & 97.43     & 1,621,650         & 1.02 MB               & 6.19 MB       & 7.22 MB\\\hline
\bottomrule
\end{tabular}
\caption{Self-ablation study on the impact of varying the number of downsampling layers.}
\label{tb:convdepth}
\end{table}

\begin{table}
\small
\centering
\begin{tabular}{c|DDD}
\toprule
$\lambda_\mu$   & Accuracy  & $\|\mu\|_2$   & $\|\mu\|_\infty < 1$  \\\hline
$ 0 $           & 93.69     & 2.743         & \ding{55}                \\
$10^{-4}$       & 95.73     & 0.821         & \ding{55}                \\
$10^{-3}$       & 97.11     & 0.132         & \ding{51}                \\
$10^{-2}$       & 97.45     & 0.028         & \ding{51}                \\
$10^{-1}$       & 96.23     & 0.007         & \ding{51}                \\
$10^{ 0}$       & 92.04     & 0.001         & \ding{51}                \\\hline
\bottomrule
\end{tabular}
\caption{Self-ablation study on the impact of varying weighting for Beltrami regularization.}
\label{tb:self}
\end{table}

\begin{table}
\small
\centering
\begin{tabular}{c|D}
\toprule
Parameterization& Accuracy\\\hline
ARAP-1            & 97.33   \\
ARAP-2            & 96.33   \\
ARAP-3            & 97.00   \\\hline
Conformal-original& 97.00   \\
Conformal-1       & 97.00   \\
Conformal-2       & 96.67   \\\hline
\bottomrule
\end{tabular}
\caption{{Self-ablation study on the impact of parameterization methods.}}
\label{tb:parameterization}
\end{table}

\newcolumntype{E}{>{\centering\arraybackslash}p{0.18\columnwidth}}
\begin{table}
\small
\centering
\begin{tabular}{c|EEE}
\toprule
Module Name & Accuracy  & Trainable Params. & Params. Size\\\hline
Estimator   & -         & 439,218           & 1.68 MB\\ 
CCNN (Classifier)  & 91.33     & 5,804,185         & 22.14 MB\\ 
CCNN-L      & 92.00     & 6,243,547         & 23.82 MB\\ 
QCCNN       & 97.00     & 6,243,403         & 23.82 MB\\\hline
\bottomrule
\end{tabular}
\caption{{Study on the impact of quasi-conformal convolution framework and purely conformal parametrized convolutional neural network.}}
\label{tb:QCCndCC}
\end{table}

\noindent\textbf{Test on Downsample Levels} We investigate the effect of the number of downsampling layers on the performance of the deformation estimation task. For this study, we use MNIST images on a single-complex test. The batch size is set to 64, and the model is trained as the manifold image classification task. The input consists of three-channel manifolds, including Gaussian curvature, mean curvature, and the image pixel channel. We evaluate the performance using different numbers of downsampling layers: 1, 2, 3, and 4 downsampling layers.

Our results indicate that using 2 downsampling layers is sufficient to achieve good performance, as reflected by the prediction accuracy in Table \ref{tb:convdepth}. Increasing the number of downsampling layers beyond 2 does not lead to significant improvements in accuracy, but it does increase computational and memory costs. These additional costs are demonstrated by several factors, including the number of trainable parameters, the forward/backward passes (storage used for all output features from each layer in the model, which indicates the amount of memory required for each pass through the network), the storage size required to store the parameters, and the overall total size of the model.

\medskip

\noindent\textbf{Test on Weighting for BC Loss} We examine the significance of the Beltrami coefficient regularization term proposed in our model. For this study, we use MNIST images on a single-complex test. Specifically, we experiment with the following values for the regularization parameter: \(\lambda_\text{bc} = 0, 10^{-4}, 10^{-3}, 10^{-2}, 10^{-1}, 1\). Table \ref{tb:self} reports the quantitative measurements of classification accuracy. The results show that the highest Dice score is achieved when the weighting for the BC loss is \(10^{-2}\), which also ensures the bijectivity of the mapping, as reflected by the infinity norm of the Beltrami coefficient. A weighting of \(10^{-3}\) does not significantly degrade the results. However, when the weighting is too small, such as \(10^{-4}\), or when the BC loss is omitted entirely, the results degrade noticeably due to the loss of bijectivity, as indicated by the infinity norm (the maximum) of the associated Beltrami coefficient. On the other hand, when the weighting exceeds \(10^{-1}\), the deformation prediction becomes over-constrained, preventing efficient training of the deformation estimator and leading to suboptimal results, even though the mapping remains bijective.

\medskip

{
\noindent\textbf{Test on Different Parametrization} To evaluate the robustness of our method with respect to the choice of initial parameterization, we conducted experiments on the craniofacial analysis task using human facial surface data. Specifically, we compared two commonly used parameterization methods: Conformal Parameterization~\cite{choi2015fast} and As-Rigid-As-Possible (ARAP) Parameterization~\cite{liu2008local}, each applied with three groups of different boundary correspondences as boundary conditions. As shown in Table~\ref{tb:parameterization}, the results across these different parameterizations are highly consistent, indicating that our method is relatively insensitive to variations in the initial conformal map.
}

\medskip

{
\noindent\textbf{Validation on Benefit of QCC} In this evaluation, we compared our QCCNN model with a model containing only the classifier baseline, which is referred to as Conformal Convolutional Neural Network (CCNN), and a larger one, CCNN-L, which has a similar number of parameters as QCCNN. Specifically, CCNN-L feeds the input to a U-Net (with the same architecture as the estimator used in QCCNN), then concatenates the U-Net output with the original input before passing it to the classifier. With such modification, CCNN-L has roughly the same parameter counts with QCCNN, and the U-Net is trained to output feature maps instead of deformation maps. As shown in Table~\ref{tb:QCCndCC}, CCNN-L achieves only marginal improvement over the original CCNN, while QCCNN consistently outperforms both significantly. This highlights the unique advantage of the learnable convolution module in QCCNN. By utilizing only a small estimator module, the model effectively adapts to complex surface data in a spatially deformable manner and gives superior results compared to similar-sized modules that output feature maps.
}
\section{Conclusion}

In this work, we introduce a novel framework for defining convolution on Riemann surfaces based on quasi-conformal theory, termed Quasi-conformal Convolution (QCC). This approach generalizes the conventional definition of convolution, enabling the convolution of functions defined on {simply-connected open} surfaces with a given kernel. By leveraging different quasi-conformal mappings, QCC dynamically adjusts convolution operations to suit specific tasks and data.

QCC facilitates adaptive and effective convolutional operations on {simply-connected open} surfaces through a dedicated module that generates data-responsive quasi-conformal mappings. This allows the definition of convolution to be dynamic and tailored to the input data. Building on this foundation, we develop the Quasi-conformal Convolutional Neural Network (QCCNN), which is designed to handle tasks involving geometric data on {simply-connected open} surfaces. By adapting the convolution process to the underlying Riemann surface structure and training data, QCCNN overcomes the challenges of defining optimal convolution operations for complex geometries, enabling the convolution definition to be learned and refined during training.

We also establish a comprehensive theoretical foundation for convolution on Riemann surfaces and the proposed QCC. This not only enhances our understanding of QCC but also provides valuable insights for constructing models tailored to specific tasks.

Our experiments demonstrate the effectiveness of QCCNN across a range of applications, including Riemann surface image classification, craniofacial analysis using 3D facial data, and facial lesion segmentation. In each case, QCC outperforms existing methods in terms of accuracy and reliability. The proposed QCC framework offers significant potential for advancing deep learning applications in fields involving complex, non-Euclidean geometric structures. Our future work is planned to extend quasi-conformal convolution to more general cases, such as high-genus surfaces.
\section*{Acknowledgments}
This work was supported by HKRGC GRF (Project ID: 14306721), and Hong Kong Centre for Cerebro- Cardiovascular Health Engineering (COCHE).

\bibliographystyle{siamplain}
\bibliography{references}

\end{document}